\documentclass[twoside,11pt]{article}
\usepackage[preprint,nohyperref]{jmlr2e}
\usepackage{ifxetex,ifluatex}
\if\ifxetex T\else\ifluatex T\else F\fi\fi T
  \usepackage{fontspec}
\else
  \usepackage[T1]{fontenc}
  \usepackage[utf8]{inputenc}
  \usepackage{lmodern}
\fi
\usepackage[dvipsnames]{xcolor}
\usepackage{amsmath}
\usepackage{tikz}
\usepackage{hyperref}
\hypersetup{
    colorlinks=true,
    linkcolor=blue,
    filecolor=magenta,      
    urlcolor=cyan,
    citecolor=blue,
    pdfpagemode=FullScreen,
}
\usepackage[a4paper, left=1in, right=1in, bottom=0.5in, includefoot=True]{geometry}
\usepackage{mathrsfs}
\usepackage{caption}
\usepackage{subcaption}
\usepackage{setspace}
\usepackage{enumitem}
\usepackage{array}
\usepackage{float}
\usepackage{tikz-cd}
\usepackage{listings}
\tikzset{every picture/.style={line width=0.75pt}}

\newenvironment{alg}{\newline\newline
\noindent
{\bf Algorithm}} 
{\hfill\newline\newline}

\newcommand{\XX}{\mathbb{X}}
\newcommand{\RR}{\mathbb{R}}

\newcommand{\Dg}{\mathrm{Dg}}
\newcommand{\reach}{\mathrm{rch}}

\newcommand{\cU}{\mathcal{U}}
\newcommand{\cV}{\mathcal{V}}

\newcommand{\Map}{\mathrm{M}_{\cU,\delta}}
\newcommand{\dbMap}{\mathrm{M}_{\cU,\delta,K}}
\newcommand{\cMap}{\mathrm{M}_{\cU}}
\newcommand{\cMMap}{\bar{\mathrm{M}}_{\cU}}
\newcommand{\cdbMap}{\mathrm{M}_{\cU,K}}
\newcommand{\cdbMMap}{\bar{\mathrm{M}}_{\cU,K}}
\newcommand{\Mapdelta}{\mathrm{M}_{\cU,K,\delta}(\XX_n, f)}
\newcommand{\Rip}{\mathrm{Rips}_\delta}

\newcommand{\MapRip}{\mathrm{M}_{r,g,K}(\Rip,f)}
\newcommand{\Ext}{\mathrm{Ext}}
\newcommand{\LZZ}{\mathrm{LZZ}}

\newcommand{\KCZZ}{\mathrm{KCZZ}}

\usepackage{xcolor}

\definecolor{codegreen}{rgb}{0,0.6,0}
\definecolor{codegray}{rgb}{0.5,0.5,0.5}
\definecolor{codepurple}{rgb}{0.58,0,0.82}
\definecolor{backcolour}{rgb}{0.95,0.95,0.95}

\lstdefinestyle{mystyle}{
    backgroundcolor=\color{backcolour},   
    commentstyle=\color{codegreen},
    keywordstyle=\color{codepurple},
    numberstyle=\tiny\color{codegray},
    stringstyle=\color{codegreen},
    basicstyle=\ttfamily\footnotesize,
    breakatwhitespace=false,         
    breaklines=true,                 
    captionpos=b,                    
    keepspaces=true,                 
    numbers=left,                    
    numbersep=5pt,                  
    showspaces=false,                
    showstringspaces=false,
    showtabs=false,                  
    tabsize=2
}
\jmlrheading{1}{2024}{1-48}{4/00}{10/00}{00-000}{Kaleb D. Ruscitti, Leland McInnes}
\ShortHeadings{Robust Mapper by Varying Resolution According to Density}{K. Ruscitti, L. McInnes}
\title{Improving Mapper's Robustness by Varying Resolution According to Lens-Space Density}

\author{Kaleb D. Ruscitti, Leland McInnes}
\author{\name Kaleb D. Ruscitti \email kaleb.ruscitti@uwaterloo.ca \\
       \addr Department of Pure Mathematics\\
       University of Waterloo\\
       Waterloo, ON, Canada
       \AND
       \name Leland McInnes \email leland.mcinnes@gmail.com \\
       \addr Tutte Institute for Mathematics and Computing\\
       Government of Canada\\
       Ottawa, ON, Canada
       }
\begin{document}

\maketitle
\begin{abstract}
    We propose a modification of the Mapper algorithm that removes the assumption of a single resolution scale across semantic space and improves the robustness of the results under change of parameters. Our work is motivated by datasets where the density in the image of the Morse-type function (the lens-space density) varies widely. For such datasets, tuning the resolution parameter of Mapper is difficult because small changes can lead to significant variations in the output. By improving the robustness of the output under these variations, our method makes it easier to tune the resolution for datasets with highly variable lens-space density.
    
    This improvement is achieved by generalising the type of permitted cover for Mapper and incorporating the lens-space density into the cover. Furthermore, we prove that for covers satisfying natural assumptions, the graph produced by Mapper still converges in bottleneck distance to the Reeb graph of the Rips complex of the data, while possibly capturing more topological features than a standard Mapper cover. Finally, we discuss implementation details and present the results of computational experiments. We also provide an accompanying reference implementation. 
\end{abstract}

\section{Introduction}

In data science, practitioners often work with a large, finite set of samples $\XX_n = \{x_1,...,x_n\}$ in $\mathbb{R}^d$, perhaps obtained by vectorizing real-world data. Examples include neural network embeddings of images or textual documents, as well as single-cell RNA expression data. Given such a set, it is useful to extract an underlying topological space $X$ from which these points are assumed to be sampled. One well-known technique for doing this is \emph{Mapper}, which employs a Morse-type function $f:X\to\mathbb{R}^l$ to construct a graph -- or more generally, a simplicial complex -- that represents $X$ \citep{singh_topological_2007}. Mapper offers several advantages; it is conceptually simple, and has been proven to converge asymptotically to the Reeb graph $R_f(X)$ in both interleaving distance and bottleneck distance \citep{munch_convergence_2016,carriere_statistical_2018}, and it produces finite graphs, which are computationally convenient to work with and can be visualised in 2 dimensions, even when $\XX_n$ lives in higher dimensions.

For these reasons, we applied Mapper to the following \emph{temporal topic modelling} application. Consider a corpus of images or textual documents that have been embedded into $\mathbb{R}^d$ through vectorization and dimension reduction, and that carry timestamps within a fixed interval. In this setting, topics or themes can be represented as clusters generated using a clustering algorithm. For a concrete example, the documents may consist of arXiv articles with submission dates, where clusters correspond to groups of articles on related subjects. Our question of interest is to understand how the clusters change over time, such as how clusters merge, split, form, or vanish. Mapper is a suitable tool to approach this problem, as with Morse-type function given by the timestamps, the graph output is naturally interpretable in terms of the changes in the topics over time.

In our initial attempts to apply Mapper to this temporal topic modelling problem, we encountered difficulties tuning the parameters for our datasets. These datasets often have widely varying frequency in different areas of $\mathbb{R}^d$. For example, in our arXiv dataset, we find that papers are submitted much more frequently in computer science topics than in number theory topics. This variation makes it difficult to chose the \emph{resolution} parameter $r$ of Mapper, which determines the scale of the topological features that Mapper can detect. Although one would want to take $r\to 0$ to capture all the features of the dataset, doing so would require infinitely many samples. In practice, there is a minimum suitable $r$ for a given set $\XX_n$, that must be chosen based on the distribution of the samples in the image of the Morse-type function. For our datasets, where time serves as the Morse-type function, the wide distribution of frequency across the dataset makes it difficult to select the resolution parameter appropriately.

In Mapper, the resolution parameter is applied uniformly across the sample set $\XX_n$. This global choice forces a compromise. A small value of $r$ may achieve high accuracy in densely sampled regions while compromising robustness in sparser areas, and a large value of $r$ does the opposite. In this article, we address this limitation by generalising Mapper to allow the resolution to vary across the dataset. Furthermore, we propose a method for computing these local variations directly from the data, thereby simplifying parameter selection for heterogeneous datasets.

There is previous research which aims to improve Mapper's parameter robustness by incorporating density information. \textit{MultiMapper} \citep{deb_multimapper_2018, dey_mutiscale_2016} varies the resolution parameter locally across the sample set by adjusting the cover of $f(X)$ rather than the cover of $X$ directly. In their approach, the resolution varies with respect to the value of the Morse function, meaning the number of degrees of freedom is equal to the dimension of $f(X)$. In contrast, our approach allows the resolution to vary in every direction inside $\RR^d$. Therefore, even for 1-dimensional Mapper we are able to vary in $d$ dimensions, getting a more robust cover of $X$. Another approach comes from the $F$- and $G$-Mapper proposals that aim to modify Mapper by making a better choice of open cover with fixed resolution \citep{alvarado_g-mapper_2023, bui_f-mapper_2020}. One could combine the strength of these methods -- first applying $G$-Mapper to select an optimal initial resolution, and then using our work to perturb $r$ locally -- and this may be an effective strategy for designing a more robust resolution selection method for Mapper. 

In Section 2, we generalise Mapper by introducing \emph{kerneled covers}, and propose variable-density kerneled covers to address the challenges described above. Then, in Section 3 we detail the reference implementation of our proposed algorithm. In Section 4, we prove that some convergence results for Mapper generalise to kerneled covers with only minor modifications. Finally, Section \ref{sec:experiments} presents the results of our computational experiments.  
 
\section{Lens-Space Density Sensitive Covers}
    \label{sec:kerneled-covers}

    Let $\XX_n=\{x_1,...,x_n\}$ denote samples taken from an unknown distribution on a topological space $X$, with known pairwise distances, and let $f:X \to \RR$ denote a Morse-type function. Our first goal is to define a class of open cover of the space $X$, that we call \emph{kerneled} covers. 

    \begin{definition}[Lens-space, Pullback Cover]
    For each $x_i\in \XX_n$, let $t_i := f(x_i)$. These all lie within the compact interval 
    $$L:=\left[\min\limits_{i=1,...n} t_i,~ \max\limits_{i=1,...,n} t_i\right] \subset f(X),$$
    that we call the \emph{lens-space}. If $f(X)$ is not compact, we replace $X$ by $X\cap f^{-1}(L)$. An open cover $\cU$ of $L$ induces the \emph{pullback cover} $f^{-1}(\cU)$ of $X$,
    \begin{equation*}
        f^{-1}\cU := \left\{ f^{-1}(U) ~|~ U\in\cU \right\}.
    \end{equation*}
    \end{definition}

  The Mapper graph associated
  to $(X,f,\cU)$ is a graph approximating the nerve $N(f^{-1}\cU)$. More concretely, we
  choose a clustering algorithm, cluster each set $f^{-1}(U_i)\cap\XX_n$, and
  label the clusters $\{{v^i}_j\}_{j=1}^{J_i}$. Then the Mapper (weighted) graph $G$ is the fully-connected graph with
  vertices given by
  \begin{equation*}
    V(G) := \bigcup_{i=1}^{N} \left\{{v^i}_j\right\}_{j=1}^{J_i},
  \end{equation*}
    with weight on the edge $({v^i}_j, {v^{i'}}_{j'})$ given by $\#\{x\in \XX_n ~|~ x\in {v^i}_j \cap {v^{i'}}_{j'}\}$. In practice, the cover $\cU$ will often be chosen so that $U_i\cap U_{i'}=\emptyset$ unless $i'=i+1$. In this case, the edge set becomes
\begin{equation*}
      E(G) = \bigcup_{i=1}^{N-1} \left\{ ({v^i}_j, {v^{i+1}}_{j'}) ~\big|~ j\in J_i, j'\in J_{i+1}\right\},
  \end{equation*}
    One limitation of using a pullback cover for Mapper, is that the resolution is determined by $\cU$, and therefore cannot vary within level sets of $f$. Kerneled covers will provide a larger class of covers that address this limitation, while preserving the good properties of Mapper.
    
    In the remainder of this section, we assume that the reader is familiar with some Morse theory as applied in the Mapper context. For a good introduction to the Morse theory used here, the reader is referred to Section 2 of \cite{carriere_structure_2018}.
\label{sec:theory}

\subsection{Kernel Perspective on Mapper}
  To help motivate the definition of kerneled covers, we start by re-writing the pullback cover of $\cU$ in terms of a kernel function. Let $\cU=\{U_i\}_{i=1}^N$ and let
  $\chi_i:L\to\{0,1\}$ be the characteristic function of $U_i$,
\begin{equation*}
    \chi_i(t) = \begin{cases}
        1, & t\in U_i,\\
        0, & t\not\in U_i.
    \end{cases}
\end{equation*}
  Then, we
  can rewrite:
  \begin{equation}
    \label{e:chiU}
    f^{-1}(U_i) = f^{-1}\left(\chi_i^{-1}(\{1\})\right).
  \end{equation}
  Every $U_i$ is an open interval with some midpoint $m_i \in \RR$ and width $w_i$, so that
  \begin{equation*}
    U_i = \left( m_i - \frac{w_i}{2},~ m_i + \frac{w_i}{2} \right),
  \end{equation*}
  and this lets us rewrite Equation \ref{e:chiU} by
  applying a square window function of width $w_i$, centered at $m_i$, to
  $\XX_n$. To be precise, for each $U_i\in \cU$, define $K_i(x):\XX_n\to\{0,1\}$ by
  \begin{equation}
  \label{e:squarekernel}
    K_i(x) := \begin{cases}
                1, & |f(x) - m_i| < w_i/2,\\
                0, & \text{otherwise.}
              \end{cases}
  \end{equation}
  Then $K_i(x) = \chi_i(f(x))$ and $f^{-1}(U_i) = K_i^{-1}(\{1\})$. The main idea behind our proposal is to replace $K_i$ with other functions, allowing us to obtain new covers that we can input into Mapper. These covers can be more general, such as incorporating local density of points in the lens-space.

  Another benefit of this framework is that it permits the construction of a cover of $X$ by fuzzy sets, which is desirable for our temporal topic modelling application. If we replace $K_i$ with a function $J_i$ that is valued in $[0,1]$, then one can interpret $J_i(x)$ as an inclusion strength of $x$ inside a fuzzy set $V_i \subset X$. However, to reduce the number of points being passed to the clustering algorithm when producing the Mapper graph, we will introduce an inclusion threshold $\epsilon \in [0,1)$. Then, we will define 
  $$V_i = J_i^{-1}\left((\epsilon, 1]\right),$$
  and under some assumptions, the set $\cV = \{V_i\}_{i=1}^N$ will form a cover of $X$. In the next section, we carefully define the class of functions $J_i$ and covers $\cV$ that have these desired properties.

\subsection{Kerneled Covers}

Now we will make precise the idea of a kernel function and its associated kerneled cover of $X$. To incorporate the local density of points in the lens-space into the cover, we want to define a kernel function that takes local density as one of its arguments. We also want these kernel functions to be centered at some value $t_0\in L$. That is, we want $$K:X\times L\times \mathbb{R}^{\geq 0} \to [0,1],$$
where the second argument is a value $t_0\in L$ denoting the midpoint or center of the kernel, and the third argument will be taken to be local density of $X$. 
To guarantee we obtain a cover of $X$, we require that the kernels
chosen have \emph{sufficient width} (Figure \ref{fig:sufficient-width}).

\begin{figure}
    \centering
    \includegraphics[width=\linewidth]{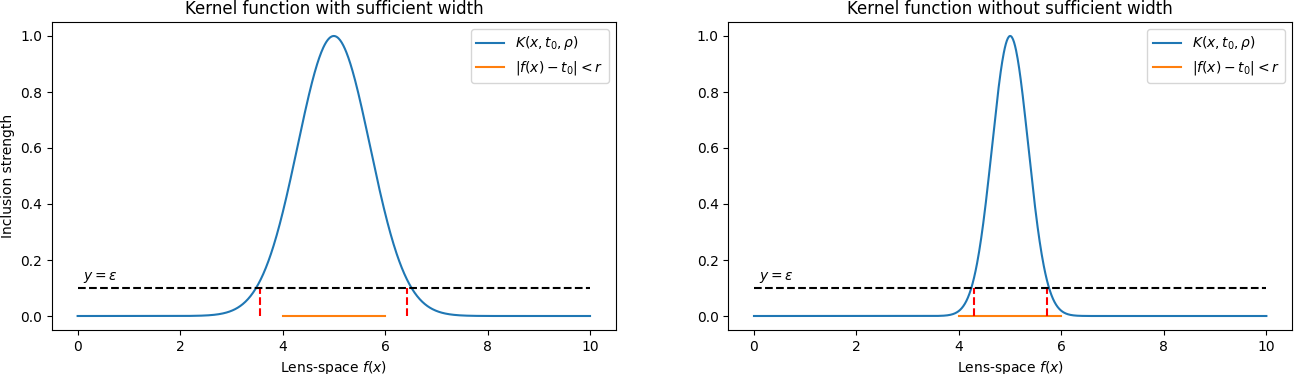}
    \caption{Plots of $K(x,t_0,\rho) \propto \exp(-[f(x)-t_0]^2/\rho^2)$, with $t_0=5$ and $\rho$ a constant. On the left, $\rho=1$ is chosen to achieve sufficient width for $(r,\epsilon)=(0.5,0.1)$. On the right, $\rho=1/4$ is chosen to not achieve sufficient width.}
    \label{fig:sufficient-width}
\end{figure}

\begin{definition}[$f$-kernel, sufficient width]
	Let $f:X\to\mathbb{R}$ be a Morse-type function and $L=f(X)$. Then an $f$-kernel function $K:X\times L \times \RR^{\geq 0} \to [0,1]$ is a  function such that:
	\begin{enumerate}
		\item For any $x$ with $f(x) = t_0$, and any $\rho>0$, one has $K(x,t_0,\rho) = 1$.
        \item Fixing $\rho>0$ and $t_0\in L$, $K(x,t_0,\rho)$ is continuous on the set $\{x\in X ~|~ K(x,t_0,\rho)>0\} .$ 
		\item For a fixed $t_0$, $K(x,t_0,\rho)$ is monotone non-increasing as $f(x)$ moves away from $t_0$.
	\end{enumerate}
    Given a pair $r>0$ and $\epsilon \in [0,1)$, we say an $f$-kernel $K$ has \emph{sufficient width} with respect to $(r,\epsilon)$, if  $|f(x)-t_0|<r \implies K(x,t_0,\rho) > \epsilon$. 
\end{definition}
\begin{example}
    Consider the square window function of Equation \ref{e:squarekernel}; for a fixed $w>0$, we can rewrite it in the form of an $f$-kernel function $K$ as
    \begin{equation*}
        K(x,t_0,\rho) := \begin{cases}
            1, & |f(x)-t_0| < w/2\\
            0, & \text{otherwise.}
        \end{cases}
    \end{equation*}
    Then for any $r\leq w/2$ and any $\epsilon\in [0,1)$, this kernel has sufficient width with respect to $(r,\epsilon)$. On the other hand, if we take $r>w/2$ then there is no $\epsilon$ for which this kernel has sufficient width. 
\end{example}

This is an unwieldly definition, so let us motivate it now. Using an $f$-kernel function $K$, we want to construct a open cover of $X$ sensitive to local density in the lens-space. To do this, we will select a set of midpoints $\{m_1,...,m_N\} \in L$, an inclusion threshold $\epsilon \in [0,1)$, and a density function $\rho:X\to\mathbb{R}^{\geq 0}$. Then, we will define 
\begin{equation*}
    V_i := \left\{
        x\in X ~|~ K(x,m_i,\rho(x)) > \epsilon
    \right\},\quad \mathrm{for}~~ i=1,...,N.
    \end{equation*}
To ensure that the collection of sets $\{V_i\}_{i=1}^N$ is indeed a cover of $X$, we will choose a resolution parameter $r>0$ and require that:
\begin{enumerate}
    \item For every $x\in X$ there is some $i$ such that  $|f(x)-m_i|<r$,
    \item $K$ has sufficient width with respect to $(r,\epsilon).$
\end{enumerate}
These requirements will ensure that $\{V_i\}_{i=1}^N$ is indeed an open cover of $X$. 

\begin{definition}
    Fix a topological space $X$ with Morse-type function $f:X\to \mathbb{R}^l$. Denote $L=f(X)$, and fix an $f$-kernel function $K$, a function $\rho:X\to \mathbb{R}^{\geq 0}$, a set $\{m_1,...,m_N\}\in L$ and $\epsilon\in(0,1]$. Then the \emph{kerneled cover} of $X$ associated to this data is the set $\cV$ consisting the open sets
    \begin{equation*}
        V_i := \{ x\in X ~|~ K(x,m_i,\rho(x)) > \epsilon \},
    \end{equation*}
    for $i=1,...,N$.
\end{definition}
\begin{proposition}
	Fix $X,f, K, \rho$ and $\epsilon$ as before. Fix a \emph{resolution} $r>0$ and a set $\{m_1,...,m_N\}\in L$ with that property that for every $x\in X$, there is some $i$ such that $|f(x)-m_{i}| < r$. Then, if $K$ has sufficient width with respect to $(r,\epsilon)$, the kerneled cover $\cV$ associated to this data is an open cover of $X$.
\end{proposition}
\begin{proof}
    Define $K_i:X\to \mathbb{R}$ as $K_i(x) = K(x,m_i,\rho(x))$. Then each $K_i$ is continuous, and therefore $V_i=K_i^{-1}(\epsilon,\infty)$ is an open set in $X$.
    Next, we will show that the sufficient width 
hypothesis guarantees that $f^{-1}(m_i-r, m_i+r)\subset V_i$. For $K$ to have sufficient width relative to $(r,\epsilon)$ it means that for every $i=1,...,N$, 
\begin{equation*}
    |f(x)-m_i|<r  \implies K_i(x)>\epsilon.
\end{equation*}
Therefore if $x\in f^{-1}(m_i-r,m_i+r)$ then $|f(x)-m_i| < r$ and hence by sufficient width, $K_i(x) > \epsilon$. 

Now because for all $x\in X$, there is an $m_i$ with $|f(x)-m_{i}|<r$, the intervals $\{(m_i-r,m_i+r)\}_{i=1}^N$ form an open cover of $L$, and thus $\{f^{-1}(m_i-r,m_i+r)\}_{i=1}^N$ is an open cover of $X$. Finally
\begin{equation*}
	X \subset \bigcup_{i\in I} f^{-1}(m_i-r,m_i+r) \subset \bigcup_{i\in I} V_i,
\end{equation*}
so $\cV$ is an open cover of $X$.
\end{proof}

To more easily compare pullback covers to kerneled covers, we also want to associate a kerneled cover to an open cover of $L$. Let $\cU=\{U_i\}_{i=1}^N$ be a generic open maximal interval cover, or \emph{gomic} of $L$, by which we mean a cover of open intervals where no more than two intervals intersect at a time and such that the \emph{overlap} $g$ of $U_j$ inside $U_i$, defined by
	\begin{equation*}
		g := \frac{\ell(U_i\cap U_j)}{\ell(U_i)}, \quad \text{ where } \ell \text{ is Lebesgue measure,}
	\end{equation*}
satisfies $g \in (0,1)$, for all $i,j$ with $U_i\cap U_j \neq \emptyset$ \cite[\S 2.5]{carriere_structure_2018}. Then each $U_i \in \cU$ is of the form 
\begin{equation*}
    U_i = \left(m_i - \frac{1}{2}\ell(U_i), ~ m_i+\frac{1}{2}\ell(U_i)\right), 
\end{equation*}
for some $m_i \in L$. This defines a set $\{m_1,...,m_N\}$ of midpoints associated to $\cU$ and a resolution $r=\max_{i=1}^N \ell(U_i)/2$. 
\begin{definition}
    Fix a cover $\cU$ of $L$ as described above, $K$ an $f$-kernel function, $\epsilon\in[0,1)$ and $\rho:X\to\mathbb{R}^{\geq 0}$. The \emph{kerneled cover} of $X$ associated to $(\cU, K, \epsilon,\rho)$ is the cover $\cV$ consisting of the sets
	\begin{equation*}
		V_i := \{x\in X ~|~ K(x,m_i,\rho(x))>\epsilon\},
	\end{equation*}
	for $i=1,...,N$. 
\end{definition}

\subsection{Varying the Cover with the Lens-Space Density}
\label{sec:lens-space-density}

\begin{figure}
\centering
\begin{minipage}{.5\textwidth}
  \centering
  \includegraphics[width=\linewidth]{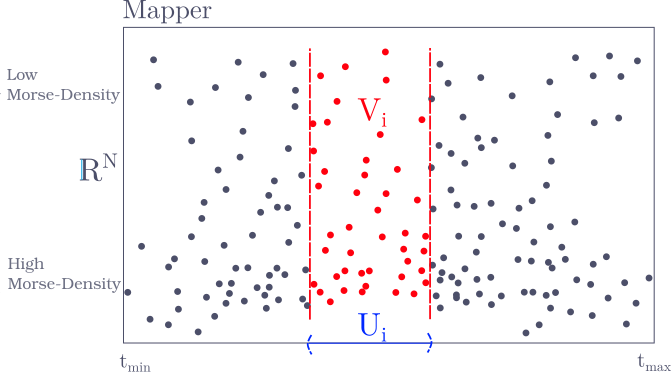}
  \captionof{figure}{Pullback open set $f^{-1}(U_i)$ of the\\ interval $U_i$, in red.}
  \label{fig:mapper-idea}
\end{minipage}%
\begin{minipage}{.5\textwidth}
  \centering
  \includegraphics[width=\linewidth]{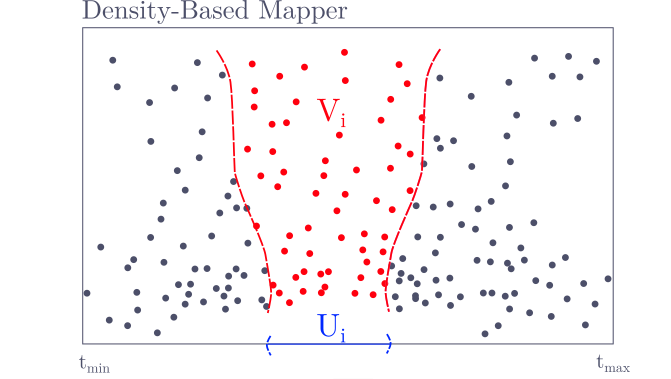}
  \captionof{figure}{Density-sensitive kerneled set corresponding to $U_i$}
  \label{fig:dbmapper-idea}
\end{minipage}
\end{figure}

With this framework, our goal is to define a kerneled cover $\cV$ whose constituent sets $V_i$ have size scaled according to the density of the points $f(\XX_n)$ inside $L$ (Figure \ref{fig:dbmapper-idea}). As $L$ is called the lens-space for the function $f$, we call the density of points in $f(\XX_n)$ the \emph{lens-space density}. This is why we included a $\rho$ dependence in the definition of an $f$-kerneled function, which remains unused so far.

As an (inverse) density estimate, we will use the distance in $L$ from $x$ to it's $k$th nearest neighbour in $X$, which we denote $\beta(x)$:
\begin{equation*}
    \beta(x) = |f(x)-f(x^{[k]})|,
\end{equation*}
where $x^{[k]}$ is the $k$-th nearest neighbour of $x$ in $X$. This can sometimes be 0, and can also get quite large, so instead of simply defining density as $\rho(x) = \frac{1}{\beta(x)}$, we want to introduce a transform $c:[0, \infty) \to [1, c_{\mathrm{max}}]$, and let $\rho(x) = c(\beta(x)).$

Although there are many reasonable choices, in the reference implementation we chose a normalized sigmoid function:
\begin{equation*}
    c(\beta) = c_{\mathrm{max}}\left(1 + \exp\left[  \frac{-(\beta-\mu)}{\sigma}\right]\right)^{-1},
\end{equation*}
where $\mu, \sigma$ are the mean and standard deviation of $\beta$ applied to $\XX_n$. This choice was made because this function is approximately linear within a standard deviation of $\mu$, meaning the width of the kernel will be most sensitive in the range where most of the data lies. However, the function is still bounded on both sides, allowing us to control the maximum and minimum densities, and therefore the maximum and minimum widths of the kernel.

The parameter $c_\mathrm{max}$ must be chosen for a given cover $\cU$, kernel $K$ and $(r,\epsilon)$ to guarantee that $K$ has sufficient width. For example, if we are using a Gaussian kernel, 
\begin{equation*}
    K(x,t_0,\rho) \propto \exp\left(-(f(x)-t_0)^2/2c(\beta(x))\right)
\end{equation*}
we can solve for a $c_\mathrm{max}$ which ensures $K(x,t_0,\rho)$ has sufficient width:
\begin{proposition}
	\label{proposition:gkernel}
	Let $\rho(x)=c(\beta(x))$ and suppose $K(x,t_0,\rho) = \exp\left[-\frac{(f(x)-t_0)^2}{2c(\beta(x))}\right]$, for some $\epsilon\geq0$ and $r>0$, with 
    \begin{equation*}
        c(\beta) = \frac{r^2}{-2\log(\epsilon)}\left(1+\exp\left[\frac{-(\beta-\mu)}{\sigma}\right]\right)^{-1}.
    \end{equation*}
    Then $K$ has sufficient width for $(r,\epsilon)$.
\end{proposition}
\begin{proof}
    Denote $\Sigma(x) = \left(1+\exp\left[\frac{-(\beta(x)-\mu)}{\sigma}\right]\right)^{-1}$ for convenience. Then suppose $x_0$ is such that $|f(x_0) - t_0| < r$. Using this inequality and
  the inequality $\Sigma(x) < 1$, we have:
	\begin{align*}
		K(x_0,t_0) &= \exp\left[ \frac{2\log(\epsilon)}{r^2} \frac{(f(x_0)-t_0)^2}{2\Sigma(x)}\right] \\
			   &> \exp\left[\frac{2\log(\epsilon)}{r^2}\frac{r^2}{2}\right]\\
			   &= \epsilon.
	\end{align*}
\end{proof}

Our motivation for allowing $K$ to be valued in $[0,1]$ is because it allows us to interpret the sets in a kerneled cover as fuzzy sets. Given a kerneled cover $\cV = \{V_i\}_{i=1}^{N}$ associated to $\{m_i\}_{i=1}^N, K,\rho,\epsilon$, we define the fuzzy sets $(V_i, K_i)$, where $K_i(x) = K(x,m_i,\rho(x))$. However, if one does not require this fuzzy set interpretation, there is always an $f$-kernel valued in $\{0,1\}$ that will generate the same kerneled cover as $K$. Thisl has the benefit of removing the choice of inclusion threshold $\epsilon >0$.
\begin{proposition}
    \label{prop:every-square}
    Let $\cV$ be the kerneled cover associated to midpoints $\{m_1,...,m_N\}\in L$, an $f$-kernel $K$, $\epsilon\in[0,1)$ and $\rho:X\to\mathbb{R}^{\geq 0}$. Let $S$ be the $f$-kernel
    \begin{equation*}
        S(x,t_0,\rho) = \begin{cases}
            1, & K(x,t_0,\rho)>\epsilon\\
            0, & \text{otherwise}.
        \end{cases}
    \end{equation*}
    Then the kerneled cover associated to $S$, with $\epsilon=0$, is equal to $\cV$
\end{proposition}
\begin{proof}
    By definition, the set $V_i\in \cV$ is $\{x\in X ~|~ K(x,m_i,\rho(x)) > \epsilon\}$. However $K(x,m_i,\rho(x))>\epsilon$ if and only if $S(x,m_i,\rho(x))>0$, and thus every set in the two kerneled covers are the same.
\end{proof}
Later on, we will need a generalization of the resolution parameter of Mapper:
\begin{definition}
	Let $\cV$ be an open cover of a topological space $X$, with Morse-type function $f:X\to L\subset \RR$. Let the \emph{resolution} of $\cV$ be 
	\begin{equation*}
		r := \sup_{V\in \cV}\left( \sup_{x,y\in V}\left|f(x)-f(y)\right| \right). 
	\end{equation*}
	When $\cV=f^{-1}(\cU)$ for a gomic $\cU$ of $L$, this reduces to the resolution for regular Mapper, $r= \sup_{ U\in \cU}\ell(U)$.
\end{definition}

\section{Implementation}

    Now we will explain how we integrated the ideas of the previous section into Mapper to design an algorithm that approximates the nerve of a kerneled cover. First, we will describe our proposed algorithm at a high-level, then elaborate on the intermediate steps. 

	Let $\XX_n \subset \RR^d$ be a set of samples and $f:\mathbb{R}^d\to \mathbb{R}$ a Morse-type function. Define $t_i := f(x_i)$, for each $x_i\in \XX_n$. Fix a clustering algorithm $\mathrm{CA}$, an $f$-kernel function $K$, and parameters $k,N\in\mathbb{N}$; $g\in (0,1)$.
\begin{alg}
	Density-based Mapper

	Input: $\left(\XX_n, \{t_1,...,t_n\}\right)$ as above, $f$-kernel $K$, parameters $N,g,k$, and a clustering algorithm $\mathrm{CA}$.
	
	Output: A weighted graph $G$.
	\begin{enumerate}[label=\arabic*.]
    \item Approximate the inverse lens-space density
      $\beta(x)$ for $x\in \XX_n$ and compute the normalized density $\rho(x)$.
		\item Compute an open cover of intervals $\cU = \{(m_i-r_i, m_i+r_i)\}_{i=1}^N$ that have overlap $g$.
		\item For each interval,
		\begin{enumerate}[label*=\alph*]
			\item Compute the kernel $K(x,m_i,\rho(x))$ for each $x\in \XX_n$,
        and define\\ $V_i = \{x ~|~ K(x, m_i,\rho(x)) > \epsilon \}$.
			\item Run the clustering algorithm CA to assign each point in $V_i$ a cluster label\\ $c_i(x)\in \{c_{i}^1,...,c_{i}^{J_i}\}$.
			\item Add vertices $\{c_{i}^1,...,c_{i}^{J_i}\}$ to $V(G)$.
		\end{enumerate}
		\item For $i=1,...,N-1$,	
		\begin{enumerate}[label*=\alph*]
			\item Compute the intersection $I_i = V_i \cap V_{i+1}$,
			\item For each point $x \in I_i$, add weight $K(x, t_i)$ to the edge between $c_i(x)$ and $c_{i+1}(x)$.
		\end{enumerate}
	\end{enumerate}
	Return $G=(V(G), E(G))$.
\end{alg}	
Detailing steps 1 and 2 will be the subject of the next two subsections. There is also a reference implementation available on GitHub:

\begin{center} \begin{verbatim} https://github.com/tutteinstitute/temporal-mapper \end{verbatim} \end{center}

\subsection{Approximating the Lens-Space Density}
	To approximate the lens-space density of $\XX_n$ near a sample $x$, we want to choose an open set $U$ around $x$, and then compute the density of points in $f(U)$.
    The choice of open sets affects the outcome. As the open sets get larger, the density is averaged across more of the dataset. On the other hand, if the open sets get too small relative to the density of the data, there may not be enough samples in each open set to get an accurate approximation. 
    For this reason, we want to choose open sets whose size varies with the density of the data. The standard technique for choosing sets with this behaviour is to use $k$ nearest neighbour distance. Choosing $U$ to be the $k$ nearest neighbour ball around $x$, the density of $f(U)$ will be $k/\ell(f(U))$, which we call the approximate lens-space density of $x$. 
    To avoid division-by-zero problems, we implement everything in terms of the \emph{inverse} approximate lens-space density. In summary:
\begin{alg} \quad

	Input: $(\XX_n, \{t_1,...,t_n\})$, and a parameter $k\in\mathbb{N}$.
	
	Output: For each $x\in \XX_n$, a value $\beta(x)\in\RR$.
	\begin{enumerate}[label=\arabic*.]
		\item For each $x_i \in \XX_n$, find the set of $k$ nearest neighbours, $\{x_{i_1},...,x_{i_k}\}$. Let $J_i := \{i_1,....,i_k\}$
		\item For each $x_i \in \XX_n$, define:
		\begin{equation*}
			\tilde{\beta}(x) = \left(\max\limits_{i_j \in J_i}(t_{i_j}) - \min\limits_{i_j \in J_i}(t_{i_j})\right)/k.
		\end{equation*}
		\item Smooth $\tilde{\beta}(x)$ by convolving with a window function $W(x,y)$;
		\begin{equation*}
			\beta(x) := \frac{1}{n}\sum\limits_{y\in \XX_n} \tilde{\beta}(y)
      W(x,y).
		\end{equation*}
	\end{enumerate}
\end{alg}
	The reference implementation uses a cosine window $$W(x,y)= \begin{cases}
	    \frac{1}{2}(1+\cos(\pi|x-y|/w)), & |x-y|\leq w \\
        0, & |x-y|>w
	\end{cases}$$
    for the last step, with width $w$ given by
  $$w^{(d)} =\left( \max\limits_{x\in \XX_n} x^{(d)}-
  \min\limits_{x\in \XX_n} x^{(d)}\right)/10,$$ 
  where $v^{(d)}$ denotes the $d$th component of the vector $v$.

\subsection{Selecting Intervals for the Open Cover of the Lens-Space}
\label{sec:covers}
    To build a kerneled cover, we will first select a set of midpoints $m_1,...,m_N$ and radii $r_1,...,r_N$. This defines a gomic $\cU$ of the lens-space $L$, containing the sets $U_i := (m_i-r_i,m_i+r_i)$. With this gomic and an $f$-kernel $K$, a kerneled cover can be defined as described in Section 2.
    
    There are multiple natural ways to select this data, and we will describe two here. Fix a number $N$ of open intervals, and an overlap parameter $g$. For any open interval $(a,b)=I_i \subset \RR$, let the midpoint be $m_i := (b-a)/2$. One choice is to take open intervals
  with fixed length; $\ell(I_i)=\ell(I_j)$ for all $i,j\in \{1,...,N\}$. In this case, the midpoints of the intervals will be evenly spaced in $L$, so we call this method \emph{Morse-spaced} cover selection. 
\begin{definition}
 \label{def:tscover}
	Let $L=[t_{\min}, t_{\max}]$, and fix $N\in\mathbb{N}$ and $g\in[0,1]$. Let $\Delta = (t_{\max} - t_{\min})/N$ and $m_i := t_{\min} + i\Delta$, for each $i \in \{1,..,n\}$. The \emph{Morse-spaced cover} with parameters $(N,g)$ is the open cover $\cU = \{I_1,...,I_N\}$ of $L$ defined by:
	\begin{equation*}
		I_i  = \left(  m_i-\frac{\Delta}{2}\left(1+\frac{g}{2}\right), \quad m_i + \frac{\Delta}{2}\left(1+\frac{g}{2}\right) \right)\cap L.
	\end{equation*}
	\vspace{0.1em}
\end{definition}
	The Morse-spaced cover has the advantage of producing graphs whose
  vertices are evenly spaced in $L$ and are therefore very interpretable in
  terms of the Morse-type function. However, for datasets where the distribution of samples is not very uniform in $L$, the intervals in a Morse-spaced cover can have a widely varying number of points in them. This can make it difficult to select appropriate parameters for the clustering algorithm $\mathrm{CA}$ used in the subsequent steps. For these datasets, we can make a mild interpretability sacrifice to mitigate this issue by choosing the intervals such that the number of data points in each set is the same. We call this \emph{data-spaced} cover selection.
\begin{definition}
	Fix $N\in\mathbb{N}$ and $g\in[0,1]$. Let $t_k = f(x_k),$ for $x_k\in \XX_n$. Sorting if necessary, suppose $t_k \leq t_{k+1}$ for all $k\in \{1,...,k\}$. 	Let $j_{i} = 1+\left\lceil\frac{i(n-1)}{N}\right\rceil$, for all $i\in\{0,...,N\}$. Then the \emph{data-spaced cover} with parameters $(N,g)$ is the open cover $\cU = \{I_1,...,I_{N}\}$ of $L$ defined by:
	\begin{equation*}
		I_i = \left( t_{j_{i-1}}\left(1-\frac{g}{2}\right),\quad t_{j_{i}}\left(1+\frac{g}{2}\right) \right).
	\end{equation*} 
\end{definition}
\begin{remark}
    The data-spaced cover can fail to be an open cover when there exists an $i$ such that $t_{j_{i-1}} = t_{j_i}$. This can happen if $N$ is very large or if there exists a region of $L$ with extremely high density.
\end{remark}

\subsection{Parameter Selection}
    When applying density-based Mapper to a dataset, one must make a choice of parameters $N, g, k$ and $f$-kernel $K$. In this section, we provide some heuristic rules to help practitioners choose reasonable values for these parameters. 

	The parameter that is most influential on the output graph is the number of midpoints $N$, which determines the \emph{resolution} $r \sim 1/N$. In the reference implementation, this is the \verb|N_checkpoints| parameter. In subsection \ref{ss:discrete-convergence} we show that the resolution determines which topological features visible in the output. For this reason, we want to make $N$ large. On the other hand, if $N$ is too large, we run into a limit imposed by the number of data points; some slices will have too few points to achieve good clustering and cluster overlaps.
    
    If the Morse function $f$ has a natural interpretation (e.g., time), practitioners can use their a priori understanding of the dataset to choose a reasonable $N$. Otherwise, $N$ should be chosen as large as permissible given your dataset size and computing constraints. One sign that $N$ is too large is when some intervals in the cover have too few points for reasonable clustering.

    \begin{remark}
    	In the reference implementation, \verb|N_checkpoints| determines the \emph{minimum} resolution of the kerneled cover. Therefore it is better to choose a larger $N$ than you would choose if using standard Mapper on the same data.
    \end{remark}
    
    For datasets that are approximately equally spaced in the lens-space, select Morse-spaced intervals. However, if there are intervals in the lens-space that contain significantly more or fewer data points, it may be difficult to  find clusters in those intervals that are consistent with the rest of the dataset. This inconsistency can produce artifacts in the resulting graph. To avoid this, you can use data-spaced intervals.
    
	The last parameter that determines the cover of the lens-space $\cU$ is the overlap parameter $g$. In the reference implementation, this is the parameter \verb|overlap|, which is scaled to lie within $(0,1)$. As with $N$, $g$ determines the \emph{minimum} overlap between intervals, so consider taking a smaller $g$ than you would with standard Mapper. In our experiments, we have found that the output graphs are fairly robust to changes in $g$. A reasonable default choice is $g=0.5$.

	Next, we have the parameter $k$, which is the number of nearest neighbours used in the  density computation. Increasing $k$ smooths out the inverse lens-space density $\beta(x)$. Moreover, the choice of $k$ is very influential on the runtime of the reference implementation; computing the $k$-nn graph is the slowest operation in density-based Mapper. For this reason, you want to take $k$ as small as you can get away with. Unfortunately, as you reduce $k$, the value of $\beta(x)$ becomes more and more sensitive to small changes in the relative position of the embedded samples $x_i$. This means that for low $k$, the results of the Mapper become less robust to changes in representation of the samples in $\mathbb{R}^d$. For datasets with 100k-1M points, we have empirically found that $100 \leq k \leq 250$ seems to be reasonable.

	Finally, there is the choice of the $f$-kernel function $K$. The simplest choice is to use square kernels whose width varies as a function of the lens-space density $\rho$. Proposition \ref{prop:every-square} demonstrates that if your clustering algorithm cannot use the value of $K$ to weigh input points, any kernel choice is equivalent to a square kernel. However, if you are using a clustering algorithm that can take weights for the points, such as \verb|scikit-learn.cluster.DBSCAN|, then the values of $K$ can be used as weights for the points to get more refined clusters. In the reference implementation, the Gaussian kernel from Proposition \ref{proposition:gkernel} is available as an alternative to the square kernel. 

	Using the reference implementation, one can recover standard Mapper by choosing a square kernel and setting the \verb|rate_sensitivity| parameter equal to 0. This will skip the density computation, and use square kernels that match the characteristic functions of the intervals in the cover of the lens-space. Doing this gives a kerneled cover equal to the pullback cover $f^{-1}(\cU)$, and the subsequent steps are the same as in standard Mapper.

\section{Convergence to the Reeb Graph}

    One appealing feature of Mapper is that its output is guaranteed to converge to the Reeb graph of the topological space in bottleneck distance, as the resolution approaches zero. In this section, we prove that this result is preserved by our proposed changes.  

    The analysis of Mapper uses the tools of persistence diagrams and zigzag modules from algebraic topology. First we will review these tools, and then generalise them to work for the kerneled covers that we introduced in Section \ref{sec:kerneled-covers}. This will allow us to put an upper bound on the bottleneck distance between a density-based Mapper graph and the Reeb graph of the underlying topological space. This bound is the content of Theorem \ref{thm:main-bound} at the end of this section.
    
\subsection{Persistence Diagrams and Bottleneck Distance}
	First, we review the definition of zigzag homology, persistence diagrams, and bottleneck distance, following \cite[\S 2.2-2.3]{carriere_structure_2018}.
    Let $X$ be a topological space, and let $f:X\to\mathbb{R}$ be Morse-type. Define the sublevel set $X^{(a)} := f^{-1}(-\infty,a]$.
    Then the family $\left\{X^{(a)}\right\}_{a\in \RR}$ defines a filtration, meaning $X^{(a)} \subseteq X^{(b)}$ for $a \leq b$. If we reverse the interval, letting $X^{(b)}_{op} = f^{-1}[b,\infty)$, we get an opposite filtration; we index these by $b\in \RR^{op}$. These filtrations can be `connected at infinity' as follows. Replace $X^{(b)}_{op}$ by the pair of spaces $(X, X^{(b)}_{op})$, and define $\RR_{\Ext} = \RR\cup\{\infty\}\cup\RR_{op}$, with order $a < \infty < \tilde{a}$ for all $a \in \RR$ and $\tilde{a}\in \RR_{op}$, and define the extended filtration to be 
\begin{equation*}
	X^{(a)}_\Ext := 
	\begin{cases}
		f^{-1}(-\infty, a], & a\in \RR\\
		X & a = \infty,\\
		\left(X, f^{-1}[a, \infty)\right), & a\in \RR_{op}\\
	\end{cases}.
\end{equation*}
Taking homology of this filtration defines the \emph{extended persistence module} $\mathrm{EP}(f)$. If the critical values of $f$ are $\{-\infty=a_0, a_1, a_2, \dots , a_n, +\infty\}$, then $\mathrm{EP}(f)$ is the sequence
\begin{equation*}
	0 \to H_\ast\left( X^{(a_0)} \right) \to \dots \to H_\ast\left(X^{(a_n)}\right) \to H_\ast(X) \to H_\ast\left(X^{(a_n)}_{\Ext}\right) \to \dots \to H_\ast\left(X^{(a_0)}_{\Ext}\right) \to 0.
\end{equation*}
When $f$ is Morse type, the extended persistence module decomposes into interval modules:
\begin{equation*}
	\mathrm{EP}(f) = \bigoplus_{k=1}^n \mathbb{I}[b_k, d_k),
\end{equation*}
where $\mathbb{I}[b_k, d_k) = \left(\RR\times[b_k, d_k)\right)\cup \left(\{0\} \times \left[\RR_\Ext - [b_k,d_k)\right]\right)$. Each summand represents the lifespan of a `homological feature', such as a connected component or non-contractible loop. The endpoints $b_k$ and $d_k$ are called the \emph{birth time} and \emph{death time} of the feature. This structure can be represented by plotting the points $(b_k,d_k)$ in $\RR^2$, and this plot is called the \emph{extended persistence diagram} of $(X,f)$ denoted $\Dg(X,f)$. 

    Next, we will define a pseudometric on persistence diagrams, called the bottleneck distance. First, we need the concept of \emph{partial matching}:
\begin{definition}
	A \emph{partial matching} between persistence diagrams $D$ and $D'$ is a subset $\Gamma$ of $D\times D'$ such that if $(p,p') \in \Gamma$ and $(p,q')\in\Gamma$ then $p'=q'$, and the same for $(p,p')$ and $(q, p')$. Furthermore, if $(p,p') \in \Gamma$, then the type (ordinary, relative, extended) of $p$ and $p'$ must match. \vspace{0em}
	
	Let $\Delta \subset \RR^2$ be the diagonal. The \emph{cost} of $\Gamma$ is
	\begin{equation*}
		\mathrm{cost}(\Gamma) := \max \left\{  \max_{p\in D} \delta_D(p), \max_{p'\in D'} \delta_{D'}(p') \right\},
	\end{equation*}
	where $\delta_D(p) = \|p-p'\|_{\infty}$ if $(p,p')\in\Gamma$ or $\delta_D(p) = \inf_{q\in\Delta}\|p-q\|_{\infty}$ otherwise.
\end{definition}
\begin{definition}[Bottleneck distance]
	The \emph{bottleneck distance} between $D$ and $D'$ is 
	\begin{equation*}
		d_\Delta(D',D') := \inf_{\Gamma} \mathrm{cost}(\Gamma),
	\end{equation*}
	ranging over all partial matchings $\Gamma$ of $D$ and $D'$.
\end{definition}
	We call $\delta_D(p)$ the transport cost of $p$. Note that if $p=(b,d)$ then 
    $$\inf_{q\in \Delta} \|p-q\|_{\infty} \leq \inf_{q\in\Delta} \|p-q\|_2 =  |b-d|.$$
    If $\delta_D(p) > \inf_{q\in \Delta} \|p-q\|_{\infty}$, then we can modify $\Gamma$ by removing $(p,p')$ to get a partial matching with lower cost. Therefore, when computing the bottleneck distance between $D$ and $D'$ we know it is bounded above by $\max\limits_{(b,d)\in D\cup D'} |b-d|$. 
	
\begin{proposition}
	Suppose $X$ is a topological space and that $f:X\to \RR$ is a Morse function. Then the endpoints of a persistence interval $[b, d)$ of $\mathrm{EP}(f)$ only occur at critical values of $f$.
\end{proposition}
\begin{proof}
	Suppose that $b \in \RR_\Ext$,  and there is no critical point $x$ with $f(x)=b$. Let $t_0$ be the largest critical value of $f$ less than $b$. Then, because $f$ is Morse-type and there are no critical points in the interval $(t_0, b)$, $X^{(b)}$ deformation retracts onto $X^{(t_0)}$. Then because $X^{(b)}$ deformation retracts to $X^{(t_0)}$, they must have isomorphic homology groups, and thus $b$ cannot be the endpoint of a persistence interval. By the contrapositive, the endpoints of persistence intervals occur only at the critical values of $f$.
\end{proof}

Zigzag persistence modules are a generalization of persistence modules, where one permits some of the arrows in the sequence to go backwards \citep{carlsson_zigzag_2010}. In particular, we want to make use of the \emph{levelset zigzag persistence module}.
\begin{definition}
	Let $f:X\to\RR$ be Morse type, and let its critical values be labeled $\{a_1,\dots,a_n\}$. Pick any set of values $\{s_i\}_{i=0}^n$ such that $a_i < s_i < a_{i+1}$, where $a_0=-\infty, a_{n+1} = \infty$. Let $X^j_i = f^{-1}([s_i,s_j])$. Then the \emph{levelset zigzag persistence module} $\LZZ(X,f)$ is the sequence:
	\begin{equation*}
		H_\ast\left(X_0^0\right) \to H_\ast\left(X_0^1\right)
    \leftarrow H_\ast\left(X_1^1\right) \to H_\ast\left(X_1^2\right) \leftarrow
    \dots \to H_\ast\left(X_{n-1}^n\right) \leftarrow H_\ast\left(X_n^n\right),
	\end{equation*}
	with linear maps induced by inclusions of the topological spaces.
\end{definition}
As with the extended persistence module, the levelset zigzag module decomposes into a sum of intervals, and the disjoint union of all those intervals is called the levelset zigzag persistence barcode, denoted $\LZZ_{bc}(X,f)$.
\begin{proposition}
  \label{proposition:pd-bc-bij}
	For $X$ a topological space with Morse type function $f$, there exists a bijection between $\Dg(X,f)$ and $\LZZ_{bc}(X,f)$.
\end{proposition}
\begin{proof}
	Corollary of the Pyramid Theorem in \cite{carlsson_zigzag_2009}.
\end{proof}
It is through the levelset zigzag that we will relate the persistence diagrams of Mapper and the Reeb graph.

\subsection{Mapper Graphs from Zigzag Modules}

Let $X$ be a topological space and suppose $\cV
= \{V\}_{i=1}^N$ is a cover of $X$. Then we can form the
zigzag module
\begin{equation}
\begin{tikzcd}
	{H_0(V_1)} && {H_0(V_2)} && {H_0(V_N)} && {} \\
	& {H_0(V_1\cap V_2)} && \dots
	\arrow["{\phi^1_{1,2}}", from=2-2, to=1-1]
	\arrow["{\phi^2_{1,2}}", from=2-2, to=1-3]
	\arrow[from=2-4, to=1-3]
	\arrow["{\phi^{N}_{N-1,N}}", from=2-4, to=1-5]
\end{tikzcd}
\label{e:zigzag}
\end{equation}
The Mapper graph associated to this cover $\cV$ is a combinatorial graph
representation of this zigzag module. For each of the `upper' vector spaces,
that is, those of the form $H_0(V_i)$, we choose a basis $\{v^i_j\}_{j=1}^{J_i}$ of the
connected components of $V_i$. For each of the `lower' vector spaces, those of
the form $H_0(V_i\cap V_{i+1})$, we choose a basis $\{e^{i,i+1}_k\}_{k=1}^{K_i}$. An element $e^{i,i+1}_k$ represents a non-empty intersection between the connected components $\phi_{i,i+1}^i(e_k^{i,i+1})$ and $\phi_{i,i+1}^{i+1}(e_k^{i,i+1})$ inside $V_i$ and $V_{i+1}$ respectively.
\begin{definition}
  The \emph{multinerve Mapper graph} $\bar{G}$ associated to the zigzag module
  (\ref{e:zigzag}) is a multigraph defined by vertex set
  \begin{equation*}
    V(\bar{G}) = \bigcup_{i=1}^N \{v^i_1,...,v^i_{J_i}\},
  \end{equation*}  
  and edge set
  \begin{equation*}
    E(\bar{G}) = \bigcup_{i=1}^{N-1} \left\{\left(\phi_{i,i+1}^i(e^{i,i+1}_k),
  ~ \phi_{i,i+1}^{i+1}(e^{i,i+1}_k)\right), k=1,...,K_i\right\}
  \end{equation*}
  The Mapper graph $G$ is the graph obtained from $\bar{G}$ by identifying all
  the parallel edges into single edges. \emph{Standard Mapper} refers to the
  Mapper graph $G$ associated to the pullback cover $f^{-1}(\cU) = \{f^{-1}(U_i) ~|~
  U_i\in \cU\}$ of a gomic $\cU$ by a Morse-type function $f$.
\end{definition}

For point cloud data, we can form a discrete approximation of this graph. Let
$\XX_n$ be a set of samples taken from a space $X\subset \RR^d$. Let
$f:X\to \RR$ be a Morse-type function, and let $\Rip(\XX_n)$ denote the Rips
complex of $\XX_n$ with radius $\delta$. Recall that \emph{single-linkage}
clustering is a clustering algorithm for samples $\XX_n$ in which the clusters
are the connected components of $\Rip(\XX_n)$. Replacing $X$ with $\Rip(\XX_n)$
and forming the (multinerve) Mapper graph yields the \emph{discrete}
(multinerve) Mapper graph. More generally, one can replace single-linkage clustering with an arbitrary
clustering algorithm, and then replace connected components of the sets $V_i$
with clusters.

We let $\Map(\XX_n, f)$ denote the discrete Mapper graph constructed with the cover 
$f^{-1}(\cU)$ and single-linkage clustering. That is, the graph with vertices given by 
\begin{equation*}
	V(G) := \bigcup\limits_{U_i \in \cU} \left\{  \text{ connected components of } \Rip\left(f^{-1}(U_i)\right)  \right\},
\end{equation*}
and edges given by adding an edge between two vertices for each $x \in \XX_n \cap (U_i\cap U_j)$ which is in both connected components.

For any continuous manifold and Morse function $X,f$, the continuous
Mapper graph with respect to the cover $f^{-1}(\cU)$ is denoted $\cMap(X,f)$. The
continous multinerve Mapper graph is denoted $\cMMap(X,f)$. If $\{K_i\}_{U_i\in \cU}$ are
kernels defining the kerneled cover $\cV$ associated to $\cU$, then the
density-based continuous Mapper graph will be denoted $\cdbMap(X,f)$, and the
density-based continuous multinerve Mapper graph will be denoted
$\cdbMMap(X,f)$.

Furthermore, these graphs are themselves topological spaces, and we can define a
Morse-type function on them derived from $f$. For any open set $V\subset \RR^d$,
define the midpoint $m_V$ to be
\begin{equation*}
  m_V := \min_{x\in V}\left(f(x)\right) + \left[
    \max_{x\in V} f(x) - \min_{x\in V} f(x)
  \right]/2.
\end{equation*}
When $I\subset \RR$ is an interval and $V = f^{-1}(I)$, $m_V$ agrees with the
midpoint of the interval $I$. Let $G$ denote any of the Mapper graphs above. By
the definition of Mapper, any vertex $v\in V(G)$ is associated to an open set
$V\subset X$. Therefore, we can define $\bar{f}:V(G) \to \RR$ by $V\to m_V$, and
extend it to a function $\bar{f}:G\to\RR$ piecewise linearly. 
 
	One advantage of Mapper is that the graphs it produces converge in bottleneck distance to the Reeb graph in the asympototic limit, $n\to \infty$. Any modifications made to the algorithm should aim to preserve this property; here we verify that the density-based changes we made preserve this limit. In this section, we must make use of the algebraic machinery of zigzag persistence. For a concise introduction to the theory used here one can look at \cite[\S2]{carriere_structure_2018}. Here we will only recall the critical definitions and results, without proof. 

	To prove the convergence of density-based Mapper, we follow the proof of the convergence for Mapper: \cite[Theorem 2.7]{carriere_statistical_2018}. There, the proof is broken into lemmas, which we now modify for density-based Mapper. One key idea is to construct a specific zigzag persistence module, called the \emph{cover zigzag persistence module} \cite[Def. 4.4]{carriere_structure_2018} which relates the persistence diagram of the Reeb graph $R_f(X)$ and the Mapper graph $\cMap(X,f)$. The proof generalizes to kerneled covers after defining the correct \emph{kerneled cover zigzag persistence module} to generalize the cover zigag persistence module.

\subsection{Kerneled Cover Zigzag Persistence}

Everywhere in this section, let $f:X\to L\subset \RR$ be Morse-type, with critical values $\{-\infty = a_0, a_1,...,a_{n+1} = +\infty\}$, indexed in increasing order, and let $\cU = \{I_i\}_{i=1}^{N}$ be a gomic of $L$. Suppose each value $a_i$ corresponds to a unique critical point $x_{c_i}$. 

To prove the bottleneck distance convergence for Mapper, Carriere and Oudot define the \emph{cover zigzag persistence module}, whose barcode encodes the persistence diagram of Mapper. First we recall their definition, and then we generalize it to kerneled covers.

\begin{definition}\cite[Def. 4.4]{carriere_structure_2018}
  For any open interval $I$ with left endpoint $a$, we define the integers $l(I)$ and $r(I)$ by{
  \begin{align*}
    l(I) &= \mathrm{max}\{i : a_i \leq a\} &
    r(I) &= \mathrm{max}\left(l(I), \mathrm{max}\{i:a_i \in I\}\right).
  \end{align*}}
  Then we define the cover zigzag persistence module $\mathrm{CZZ}(f, \cU)$ by
  \begin{equation*}
     H_\ast \left(X^{r(I_1)}_{l(I_1)}\right) \leftarrow \dots \to H_\ast\left(X^{r(I_k)}_{l(I_k)}\right) \leftarrow H_\ast\left(X_{l(I_k\cap I_{k+1})}^{r(I_k\cap I_{k+1})}\right) \to H_\ast\left(X^{r(I_{k+1})}_{l(I_{k+1})}\right) \leftarrow \dots \to  H_\ast \left(X^{r(I_N)}_{l(I_N}\right) 
  \end{equation*}
\end{definition}
\begin{proposition}\cite[Lemma 4.5]{carriere_structure_2018} 
  With $f:X\to L$, and open cover $\cU$ as above, there is a bijection between $\Dg(\cMMap, \bar{f})$ and $\mathrm{CZZ}_{bc,0}(f,\cU)$.
\end{proposition}

We generalize:

\begin{definition}
    \label{def:kczz}
	Let $\cU = \{U_i\}_{i\in[N]}$ be a gomic of $L$, and let $\cV = \{V_i\}_{i\in [n]}$ be a kerneled cover of $X$ associated to $\cU$. The \emph{kerneled cover zigzag persistence module} $\mathrm{KCZZ}(f,\cV)$ is 
	\begin{equation*}
		H_\ast(V_1) \leftarrow \dots \to  H_\ast(V_k) \leftarrow H_\ast(V_k\cap V_{k+1})
    \to H_\ast(V_{k+1}) \leftarrow \dots \to H_\ast(V_{N})
	\end{equation*}
\end{definition}
\begin{remark}
    If a set $V_k$ does not contain a critical point of $f$, then there is a deformation retract onto the previous set $V_{k-1}$. This gives an isomorphism:
    \begin{equation*}
        H_\ast(V_{k-1}) \cong H_\ast(V_{k-1}\cap V_k) \cong H_\ast (V_k).
    \end{equation*}
    In the remainder of this section, we will assume that we've reduced the kerneled cover zigzag persistence module by removing all such isomorphic terms. In particular, we assume therefore that each $V_k$ contains a critical point of $f$.
\end{remark}

For a square kernel, where $\cV$ is the pullback cover, this is isomorphic to the cover zigzag persistence module $\mathrm{CZZ}(f,\cU)$.
\begin{proposition}
	Let $X$ be a topological space with Morse-type function $f$, and with open covers $\cU$ and $\cV$ as in Definition \ref{def:kczz} above. Then
  $\mathrm{LZZ}_0(\cdbMMap,\bar{f})$ and $\mathrm{KCZZ}(f,\cV)$ are equal as zigzag persistence modules.
\end{proposition}
\begin{proof}

    Let $m$ be the number of critical points of $(\cdbMMap, \bar{f})$, so that $\mathrm{LZZ}_0(\cMMap, \bar{f})$ contains $2m+1$ terms. By the definition of $\bar{f}$, these can occur only at vertices of $\cdbMMap$, which are connected components of open sets in $\cV$. Therefore, if $\bar{f}$ has a critical point at a vertex $v^i_j$ inside set $V_i$, then $V_i$ contains a critical point of $f$. This implies that there are at least $m$ many $V_i$ containing critical points, and therefore $\mathrm{KCZZ}(f,\cV)$ has at least $2m+1$ terms. 

    However, some of these terms will be isomorphic. Suppose that $V_i$ contains a critical point which is inside a connected component ${v^i}_j$, but that ${v^i}_j$ is not a critical point of $\bar{f}$. Then this implies that there exists a deformation retract of $\cdbMMap$ to the previous critical point of $\bar{f}$, ordered in terms of $\bar{f}$, say ${v^{k}}_l$. This deformation retract induces an isomorphism $H_\ast (V_j) \cong H_\ast(V_i)$ for $k\leq j\leq i$. If $k=i$ then $V_i$ is one of the $m$ sets counted above, and otherwise it is isomorphic to one. Therefore the number of non-isomorphic terms in $\mathrm{KCZZ}(f,\cV)$ is $2m+1$. Now it remains to show that the $2m+1$ terms of $\mathrm{KCZZ}(f,\cV)$ and $\mathrm{LZZ}_0(\cMMap, \bar{f})$ are equal. 
    
  The even terms of the zeroth dimension levelset zigzag have the form:
  \begin{equation*}
    H_0(X_i^{j+1}) = H_0\left(\bar{f}^{-1}[s_i,s_{i+1}]\right).
  \end{equation*}
  By definition, there is one critical value of $\bar{f}$, $a_{i+1}$, in the
  interval $[s_i, s_{i+1}]$. Critical values of $\bar{f}$ occur only at vertices
  of $\cdbMMap(X,f)$; let $v_i$ be the vertex with $\bar{f}(v_i) = a_{i+1}$. By
  definition of $\cdbMMap(X,f)$, $v_i$ is a connected component of 
  $V_{j_i} \subset X$ for some $j_i\in [N]$, and all the connected components of
  $V_{j_i}$ define vertices with $\bar{f} = a_i$. Therefore:
  \begin{align*}
    \dim H_0\left( \bar{f}^{-1}[s_i, s_{i+1}]\right) &= \#\{v\in
      V(\cMMap(X,f))~|~\bar{f}(v) = a_i\} \\
    &= \#\{\text{connected components of } V_{j_i}\} \\
    &=\dim H_0(V_{j_i}).
  \end{align*}
  Moreover, since $\KCZZ(f,\cV)$ only includes $V_k$ that contain a critical
  point, the value $a_{i+1}$ will be associated to a vertex in $V_{i+1}$.

  The odd terms of the levelset zigzag have the form
  \begin{equation*}
    H_0(X_i^i) = H_0\left(\bar{f}^{-1}(s_i)\right).
  \end{equation*}
  By construction, $s_i$ is not a critical value of $\bar{f}$. Suppose there is
  some vertex $v$ of $\cdbMMap(X,f)$ with $\bar{f}(v) = s_i$. Then $v$ is not a
  critical point, and hence $f^{-1}(s_i)$ is homotopy equivalent to $f^{-1}(s_i
  - \epsilon)$, for all sufficiently small $\epsilon > 0$. Hence we can, up to
  replacing $s_i$ by $s_i-\epsilon$, assume that there is no vertex with
  $\bar{f}(v) = s_i$. Following the argument from the even case, the critical
  values $a_i$ and $a_{i+1}$ can be associated to sets $V_i$ and $V_{i+1}$ in
  the cover $\cV$. Since $\bar{f}$ is defined by piecewise linear extension, the
  points in $\cdbMMap(X,f)$ which have a value of $s_i$ will lie on 
   edges connecting vertices associated to connected components in $V_i$ and
  $V_{i+1}$. Thus,
  \begin{align*}
    \dim H_0\left(\bar{f}^{-1}(s_i)\right) &= \#\{\text{edges between connected
    components of } V_i \text{ and } V_{i+1}\}\\ 
    &= \dim H_0(V_i\cap V_{i+1}).
  \end{align*}
  Thus each term of these two zigzag modules are isomorphic.

\end{proof}
\begin{proposition}
  \label{thm:dg-kczz}
  There exists a bijection between $\Dg(\cdbMMap, \bar{f})$ and
  $\KCZZ_{cb,0}(f, \cV)$.
\end{proposition}
\begin{proof}
  Combine the previous result with Prop. \ref{proposition:pd-bc-bij}.
\end{proof}

\subsection{Convergence of Continuous Density-Based Mapper}
Recall that our goal is to relate the diagram $\Dg(\cdbMMap, \bar{f})$ with the persistence
diagram of the Reeb graph, $\Dg(R_f(X), f)$. We have related $\Dg(\cdbMMap,
\bar{f})$ with the kerneled cover zigzag module, and so next we must relate
$\KCZZ(f,\cV)$ to $\Dg(R_f(X), f)$.

For standard multinerve Mapper, this relationship is given by the next theorem.
\begin{theorem}
  \label{thm:mmapperconv}
  Let $X$ be a topological space, and $f:X\to\RR$ a Morse-type function. Let
  $R_f(X)$ be the Reeb graph, and by abuse of notation, $f:R_f(X) \to \RR$ the
  induced map. Let $\cU$ be a gomic of $f(X)$, with resolution $r$. Then 
  \begin{equation*}
  \Dg(R_f(X),f) - \{(x,y) ~|~ |y-x| \leq r\}\subset  \Dg(\cMMap(X,f)) \subset \Dg(R_f(X), f)
  \end{equation*}
\end{theorem}
\begin{proof}
  This is the first half of Corollary 4.6 of \cite{carriere_structure_2018}.
\end{proof}
\begin{remark}
    Theorem \ref{thm:mmapperconv} for Multinerve Mapper reduces to Mapper as well. Consider the projection $\pi:\cMMap(X,f) \to \cMap(X,f)$ by identifying all edges connecting the same pair of vertices. For 1d Mappers, this projection induces a surjection in homology \cite[Lemma 3.6]{carriere_structure_2018}. Using this, one can show
    $$\Dg(R_f(X),f) - \{(x,y) ~|~ |y-x| \leq r\}\subset \Dg(\cMap(X,f)) \subset \Dg(\cMMap(X,f)) \subset \Dg(R_f(X), f).$$
    The same argument applies when substituting in a kerneled cover. 
\end{remark}
Given a kerneled cover $\cV$ associated to an open cover $\cU$ of $L$, we can define two pullback covers:
\begin{definition}
	Let $X$ be a continuous manifold, $f:X\to L\subset \RR$ a Morse-type function,
  and $K$ an $f$-kernel of sufficient width with respect to some cover $\cU$ of
  $L$. Let $\cV$ be the cover associated to $(\cU, K)$. The
  \emph{maximally-coarse} cover $\cU^c$  of $L$ associated to this data is the
  set consisting of
	\begin{equation*}
		U^c_i := \left( \inf_{x\in V_i} f(x) ~,~ \sup_{y\in V_i} f(x)\right) 
	\end{equation*}
	for all $V_i \in \cV$, and we let $\cV^c$ denote its pullback cover of $X$.
  Analogously, the \emph{maximally-fine} cover $\cU^f$ of $L$ associated to this
  data is the set consisting of
	\begin{equation*}
		U^f_i := \left( \sup_{x\in V_i} f(x) ~,~ \inf_{y\in V_i} f(x)\right) 
	\end{equation*}
	for all $V_i\in \cV$, and we let $\cV^f$ denote its pullback cover of $X$.
\end{definition}

\begin{definition}
    Let $X$ be a topological space with Morse-function $f:X\to \mathbb{R}$. Here we call a finite open cover $\cV=\{V_i\}_{i=1}^{N}$ of $X$ \emph{regular}, if 
    \begin{enumerate}
        \item For all $1\leq i\leq N$, $\ell(V_i \cap V_{i+1}) > 0$, where $\ell$ denotes Lebesgue measure.
        \item For all $i\neq j \neq k$, $V_i\cap V_j\cap V_k = \emptyset$.
        \item $\cV$ has no proper subcovers.
    \end{enumerate}
    This is engineered so that the maximally coarse and fine covers $\cU^c$ and $\cU^f$ associated to $\cV$ will be gomics.
\end{definition}

Then, for $f:X\to L$, with a choice of $\cU$ and $K$, we have three covers of $X$, $\cV, \cV^f$ and $\cV^c$, which define three continuous Mappers of $X$, which we denote $\cdbMap(X,f)$, $\cMap^f(X,f)$ and $\cMap^c(X,f)$. Let $r_f$ and $r_c$ denote the resolution of $\cV^f$ and $\cV^c$ respectively. 
\begin{proposition}
	Suppose that $\cV=\{V_i\}$ and $\tilde{\cV} =\tilde{V}_i$ are two regular covers of $L$ with the same number of sets, and that $V_i \subset \tilde{V}_i$ for every $i$.  Then the number of  path connected components of $\tilde{V}_i$ is less than or equal to the number of path connected components of $V_i$ for every $i$.
\end{proposition}
\begin{proof}
	 Let $C$ be a path connected component of $\tilde{V}_i$. Then either $C\cap V_i$ is connected, or it has connected components $C_1,...,C_K$. If $C_i$ and $D_j$ are connected components of $\tilde{V}_i$, then $C_i\cap V_i$ and $D_j\cap V_i$ must not be in the same connected component of $V_i$: Suppose for contradiction that they were in the same connected component $E$. Then there is a path $\gamma$ from some point $p_c \in C_i \cap V_i$ to some point $p_d \in D_j \cap V_i$. However, since $\gamma \subset E \subset V_i \subset\tilde{V}_i$, this path would also connect $C$ and $D$, contradicting the hypothesis.
\end{proof}
In particular, for any kerneled cover $\cV$, we have that $V_i^f \subset V_i \subset V^c_i$ for all $i$.

\begin{definition}
	Suppose that $\cV=\{V_i\}$ and $\tilde{\cV} =\tilde{V}_i$ are two regular
  covers of $L$ with the same number of sets, and that $V_i \subset \tilde{V}_i$
  for every $i$. Then the \emph{coarse-node identification map}
  $\phi:\cdbMap(X,f) \to \cMap(X,f)$ is the graph morphism defined as follows.
  Any vertex $v \in \cdbMap(x,f)$ is a connected component of an open set $V_i
  \in \cV$. There is a corresponding $\tilde{V}_i\supset V_i$ whose connected
  components are the vertices of $\cMap(X,f)$. Let $\phi(v)$ be the connected
  component of $v$ in $\tilde{V}_i$.
If $e=(u,v)$ is an edge of $\cdbMap(X,f)$, then let $\phi(e) = \left(\phi(u), \phi(v)\right)$.
\end{definition}
\begin{proposition}
	The coarse-node identification map $\phi$ is well-defined and surjective. 
\end{proposition}
\begin{proof}

	First, for well-definedness, let $e=(v_1,v_2)$ be an edge of $\mathrm{M}_K$. Then there exists a point $x\in v_1\cap v_2 \subset V_{i}\cap V_{i+1}$. As $v_1\subset \phi(v_1)$ and $v_2\subset \phi(v_2)$, we have $x\in \phi(v_1)\cap \phi(v_2)$ and thus the edge $\phi(e)$ exists in $\mathrm{M}$, hence $\phi$ is well defined. Next we prove that $\phi$ is surjective.

	Suppose $a\in \RR$ is a critical value of $f$. Then, because $V_i \subset \tilde{V}_i$, the number of connected components of $\tilde{V}_i$ is less than or equal to the number of connected components of $V_i$. Therefore, the number of vertices of the graph $\mathrm{M}$ associated to $a$ is less than or equal to the number of vertices of the graph $\mathrm{M}_K$ associated to $a$. Therefore, at each such value, the map $\phi:V_a(\mathrm{M}_K) \to V_a(\mathrm{M})$ by taking a connected component in $V_i$ to its connected component in $\tilde{V}_i$ is surjective.
	
	Suppose $e=(w_1,w_2)$ is an edge of $\mathrm{M}$. Then there exists a point $y$ in $w_1\cap w_2 \subset \tilde{V}_i\cap\tilde{V}_{i+1}$. It is possible that $y\not \in V_i \cap V_{i+1}$, but since $\cV$ is a cover, either $y\in V_i$ or $y\in V_{i+1}$. WLOG, suppose $y\in V_i$. Then let $v_1$ be the connected component of $y$ in $V_i$, and let $v_2$ be $v_1 \cap V_{i+1}$. It remains to show that $v_2 \neq \emptyset$. 

	Since $y\in w_1\cap w_2$, we have that $w_1\cap w_2 \neq \emptyset$ and thus $w_1$ and $w_2$ belong to the same connected component $C$ inside $\tilde{V}_i\cup \tilde{V}_{i+1}$. Let $A = V_i \cap C$ and $B = V_{i+1} \cap C$. Then $A\cup B = C$,  $A\cap w_1 = v_1$ and $B \cap w_2 = v_2$. Furthermore $A\cap B = V_i\cap V_{i+1} \cap C = v_1 \cap V_{i+1} = v_2$. Therefore if $v_2 = \emptyset$, then $A$ and $B$ would disconnect $C$, a contradiction.
\end{proof}
The coarse-node identification map between the maximally coarse Mapper and
density-based Mapper can be interpreted as a map of modules from $\KCZZ(f,\cV)
\to \mathrm{CZZ}(f,U^c)$, as we now explain. The slice $f^{-1}(I_i)$ is homotopy
equivalent to $X_{l(I_i)}^{r(I_i)}$ for every $I_i\in \cU^c$, and thus
$H_0(X_{l(I_i)}^{r(I_i)})$ is generated by the connected components of
$f^{-1}(I_i)$, which are vertices $\{v^i_j\}$ of the coarse Mapper. On the other
hand, $H_0(V_i)$ is generated by the connected components of $V_i$, which are
vertices of the density-based Mapper, which surject via $\phi$ onto
$H_0(X_{l(I_i)}^{r(I_i)})$.

Thus, we have the following:
\begin{equation*}
  \label{e:subset}
  \Dg(\cMMap, \bar{f})^c \subset \KCZZ_{cb,0}(f,\cV) \subset \Dg(\cMMap,
  \bar{f})^f.
\end{equation*}
Combining Equation \ref{e:subset}, Theorem \ref{thm:mmapperconv} and Proposition
\ref{thm:dg-kczz}, we can now obtain the convergence result for continuous
density based Mapper.

\begin{theorem}
  Let $X$ be a topological space, and $f:X\to \RR$ a Morse-type function. Let
  $R_f(X)$ be the Reeb graph, and let $\cV$ be a cover of $X$ with maximally
  coarse resolution $r$. Then
  \begin{equation*}
  \Dg(R_f(X), f) - \{(x,y) ~|~ |y-x| \leq r \} \subset \Dg(\cdbMMap(X,f), \bar{f}) \subset
  \Dg(R_f(X), f).
\end{equation*}
  Moreover, the persistence diagrams become equal if $r$ is smaller than the
  smallest difference between distinct critical values of $f$.
\end{theorem}

\begin{proposition}
	\label{lemma2}
	Let $\cV$ be a kerneled cover of $(X,f)$ associated to $(\cU,K)$ and let $\tilde{\cU}$ be the associated maximally-coarse cover. Then 
	\begin{equation*}
    d_\Delta\left(\Dg(R_f(X)), \Dg(\cdbMap, \bar{f})\right) \leq r,
	\end{equation*}
	where $r$ is the resolution of $\tilde{\cU}$.
\end{proposition}
\begin{proof}	
  This is a corollary of the inclusion result above. 
\end{proof}
\subsection{Convergence of Discrete Density-Based Mapper}
\label{ss:discrete-convergence}
Now that we know density-based Mapper converges to the Reeb graph when we have
access to a bona fide continous manifold $X$, it remains to understand what
happens when we work with samples $\XX_n \leftarrow X$.

First, we generalize a result about the Mapper of the the Rips complex to kerneled covers from \cite{dey_reeb_2013}.
\begin{definition}
    Let $G = (\XX_n, E)$ be a graph with vertices $\XX_n$. For $e=(x,x') \in E$, let $I(e)$ be the interval 
    \begin{equation*}
        \left( \mathrm{min}(f(x),f(x')), ~ \mathrm{max}(f(x), f(x')) \right).
    \end{equation*}
    Then $e$ is said to be $\emph{intersection-crossing}$ for the gomic $\cU$ if there is a pair of consecutive intervals $U_i, U_j \in \cU$ such that $U_i\cap U_j \subset I(e)$.
\end{definition}

\begin{proposition}
\label{lemma1}
	Let $\mathrm{Rips}_\delta^1(\XX_n)$ denote the 1-skeleton of $\Rip(\XX_n)$. If it has no intersection-crossing edges, then $\dbMap(\XX_n,f)$ and $\MapRip$ are isomorphic as combinatorial graphs.
\end{proposition}
\begin{proof}
	Let $L\subset \RR$ denote the image of $f(X)$, and let $\cU$ be the open cover of $L$ used in the density-based Mapper algorithm. For each $U_i \in \cU$, let
	$K_i(x) = K(x,t_i)$ for the midpoint $t_i\in \cU$
	
	A vertex of $\MapRip$ is a connected component of $\mathrm{Rips}_\delta(\XX_n)\cap V_i$, for some $V_i = K_i^{-1}(\epsilon,\infty)$. Denote these connected components as $c_{i}^j$, $j=1,...,{J_i}$. The connected components of $\mathrm{Rips}_\delta(\XX_n)\cap V_i$ are the same as the connected components of its 1-skeleton, which are the vertices of $\Mapdelta$ by definition. Therefore the set of vertices $\{c_{i}^j\}$ of these graphs are the same.
	
	It remains to construct an isomorphism on the edges. An edge $(c_{i_1}^j, c_{i_2}^k)$ of $\Mapdelta$ corresponds to a point $x \in \XX_n$ which is in the intersection of two pullbacks; $x\in V_{i_1}\cap V_{i_2}$, and which is in the connected component $c_{i_1}^j \subset V_{i_1}$ and $c_{i_2}^k \subset V_{i_2}$.  On the other hand, an edge $(c_{i_1}^j, c_{i_2}^k)$ of $\MapRip$ is a point $x\in X$ with the same propositionerty. This shows the edges of $\Mapdelta$ are a subset of the edges of $\MapRip$, but we still need to check the other inclusion.
	
	By regularity of the cover $\cU$, the intersection $U_{i_1}\cap U_{i_2}$ is an interval $I$. The existence of an edge in $\MapRip$ implies that the components $c_{i_1}^{j}$ and $c_{i_2}^k$ are connected in $V_{i_1}\cup V_{i_2}$; call this connected component $C$. Since $C$ is connected, its 1-skeleton is connected, so there is a 1-cell $\tilde{e}$ connecting $c_{i_1}^j$ and $c_{i_2}^k$. Let its endpoints be denoted $u$ and $v$, which are both points in $\XX_n$. If $f(u)\in I$, then $u$ will define an edge in $\Mapdelta$, and same with $v$. Otherwise, the edge $\tilde{e}$ is intersection-crossing. Thus if $\mathrm{Rips}_\delta^1(\XX_n)$ has no intersection-crossing edges, then the set of edges of $\Mapdelta$ and $\MapRip$ are equal.
\end{proof}

	Using these lemmas, one can recover the convergence result:

\begin{definition}[Reach and Convexity Radius]
    Let $X\subset\RR^d$ be a topological space. The \emph{medial axis} of $X$ is the set of points in $\RR^d$ with at least two nearest neighbours in $X$:
    \begin{equation*}
        \mathrm{med}(X) := \left\{y \in \RR^d ~|~ 
        \#\{x\in X ~|~ \|y-x\| = \|y,X\| \} \geq 2\right\},
    \end{equation*}
    where $\|y,X\| = \inf\{\|y-x\| ~|~ x\in X\}$. The \emph{reach} of $X$, $\reach$ is the distance $\inf\{\|x-m\| ~|~ x\in X, m\in \mathrm{med}(X)\}.$ \vspace{1em}

    The \emph{convexity radius} of $X$ is the supremum of $\rho\in\RR$ for which every ball in $X$ of radius less than $\rho$ is convex.
\end{definition}
\begin{definition}[Modulus of Continuity]
    Let $f:X\to\RR$ be a Morse-type function. The modulus of continuty $\omega_f$ of $f$ is the function $\RR_{\geq 0}\to\RR_{\geq 0}$,
    \begin{equation*}
        \omega_f(\delta) = \sup\{
            |f(x) - f(x')| ~|~ x,x'\in X,~ \|x-x'\|\leq \delta
        \}
    \end{equation*}
\end{definition}
\begin{theorem}
    \label{thm:main-bound}
	Suppose $X$ is a compact continuous manifold with positive reach $\reach$ and convexity radius $\rho$. Let $\XX_n\leftarrow X$ be $n$ points sampled from $X$. Let $f$ be a Morse-type function on $X$ with modulus of continuity $\omega$. Then if the three following hypotheses hold:
	\begin{enumerate}
		\item $\delta \leq \frac{1}{4} \min\{\reach,\rho\}$ and $\omega(\delta) \leq \frac{1}{2}e_{\mathrm{min}}$,
		\item $\max\{ |f(x_i) - f(x_j)|  ~|~ x_i,x_j \in \XX_n, \text{ such that } \|x_i-x_j\| \leq \delta \} \leq gr$,
		\item $4d_H(X,\XX_n) \leq \delta$, 
	\end{enumerate}
	Then the bottleneck distance between the density-based Mapper with single-linkage clustering and the Reeb graph satisfies
	\begin{equation*}
		d_\Delta(R_f(X), \Mapdelta) \leq r + 2\omega(\delta).
	\end{equation*}
	\vspace{0.1em}
\end{theorem}
\begin{proof}
	The proof proceeds exactly as the proof of Theorem 2.7 in \cite{carriere_statistical_2018}, replacing their Theorems A2 and A4 with our Propositions \ref{lemma1} and \ref{lemma2}. We provide a sketch here. As shorthand, let $\dbMap = \dbMap(\XX_n, f).$

    \begin{align*}
        d_\Delta(R_f(X), \dbMap) &= d_\Delta(\Dg(R_f(X), f),~\Dg(\dbMap, f))\\
        &= d_\Delta(\Dg(R_f(X), f), ~ \Dg(\cdbMap(\Rip(\XX_n), \bar{f})))\\
        (\text{triangle inequality}) \quad &\leq d_\Delta\left(\Dg(R_f(X), f),~\Dg(R_{\bar{f}}(\Rip(\XX_n), \bar{f}))\right) \\
        &+ d_\Delta\left(\Dg(R_{\bar{f}}(\Rip(\XX_n), \bar{f})),~ \Dg(\cdbMap(\Rip(\XX_n), \bar{f}),\bar{f})\right)\\
        &\leq 2\omega(\delta) + r/2
    \end{align*}
    The equality from first to second line is due to Proposition \ref{lemma1}, which holds due to hypothesis 2), and the final equality is due to hypotheses 1) and 3), and Proposition \ref{lemma2}.
\end{proof}	

Theorem \ref{thm:main-bound} tells us that density-based Mapper will perform at least as well (in bottleneck distance) as regular Mapper for the same gomic $\cU$. Furthermore, it allows one to apply the statistical analysis of Mapper by Carriere, Michel and Oudot \cite[\S 3-4]{carriere_statistical_2018} to density-based Mapper.

\section{Computational Experiments}
\label{sec:experiments}

In this section, we explain two computational experiments we conducted to validate density-based Mapper. The first experiment aims to verify that our proposed method for approximating lens-space density and computing a kerneled cover matches our theoretical expectations. Specifically, we want to recreate Figure \ref{fig:dbmapper-idea} computationally using synthetic data. The second experiment qualitatively assesses whether density-based Mapper addresses Mapper's difficulty capturing topological features in datasets with largely varying lens-space density. 

\subsection{Verifying the Approximate Lens-Space Density Computation}
\begin{figure}
    \centering
    \includegraphics[width=0.4\textwidth]{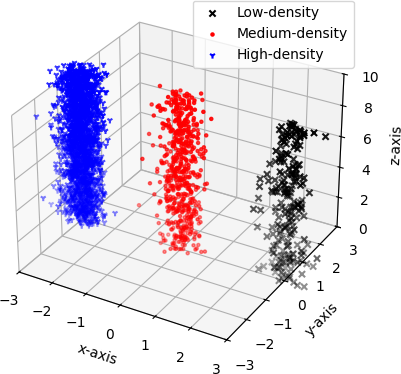}
    \caption{Synthetic data for lens-space density computation.}%
    \label{fig:data-pullback}%
\end{figure}
To verify that our construction of kerneled covers matches the theoretical intuition described in Section \ref{sec:theory}, we first constructed a dataset with three distinct components, each having a different lens-space density. According to our proposal for building a kerneled cover, we predict a cover comprising sets whose width increases as the lens-space density decreases.

Our test dataset is a subset of $\mathbb{R}^3$, with Morse-type function $f$ defined as projection to the $z$-axis. We generated samples $\XX_n := \{x_i,...,x_n\}$ by sampling from a mixture of three distributions, with the mixture coefficients controlling the relative lens-space density of the components. Each component distribution combined a circular Gaussian distribution for the $x$ and $y$ coordinates with a uniform random distribution over $[0,10]$ for the $z$-coordinate. In particular, this means the lens-space is $L=[0,10]$. The generated data is shown in Figure \ref{fig:data-pullback}, where the components are distinguished by colour. Python code used to generate this data and subsequent figures is provided in Appendix A.

We computed the approximate lens-space density and a lens-spaced kerneled cover, as described in Section \ref{sec:covers}. To visualize the result in 2 dimensions, we reduced the $(x,y)$-axes in the data to a single axis using principal component analysis. The result is shown in Figure \ref{fig:pullback}. 

\begin{figure}[h]
	\centering
	\includegraphics[width=\textwidth]{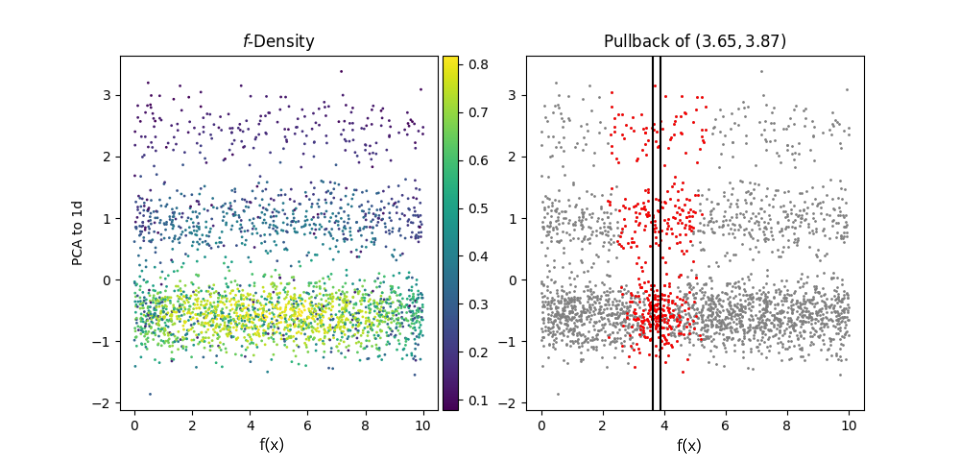}
	\caption{Left: Colour represents the lens-space density. \quad Right: Points in the kerneled pullback are red.}
	\label{fig:pullback}
\end{figure}

The left-side plot displays the computed lens-space density using colour, normalized to the interval $(0,1)$ as explained in Section \ref{sec:lens-space-density}. Three distinct regions – corresponding to high, medium, and low lens-space density – are visible, matching our intended design.

The right-side plot displays the open set $K^{-1}(0.1, \infty)$ from the kerneled cover associated to the lens-space interval $I=(3.65, 3.87)$. Points coloured red belong to this set, while points between the black lines are in $f^{-1}(I)$, the pullback open set used in standard Mapper. As predicted, the kerneled cover's open sets enlarge in regions of lower lens-space density and contract in denser regions. The similarity of Figure \ref{fig:dbmapper-idea} and Figure \ref{fig:pullback} provides evidence that our implementation and theoretical intuition align. 

\subsection{Comparison of Mapper and Density-Based Mapper on Variable-Density data}

The motivation behind density-based Mapper is to improve Mapper's performance on datasets with varying lens-space density. The goal is for density-based Mapper to detect small features in high-density regions of the dataset without fragmenting large features in sparse regions due to insufficient data. 

To evaluate density-based Mapper's effectiveness, we constructed a minimal example dataset with variable lens-space density that is challenging for Mapper. We proceeded to run Mapper and density-based Mapper on this dataset with a range of parameters, to assess the impact of our proposed changes on the results.

\begin{figure}
    \centering
    \includegraphics[width=0.75\textwidth]{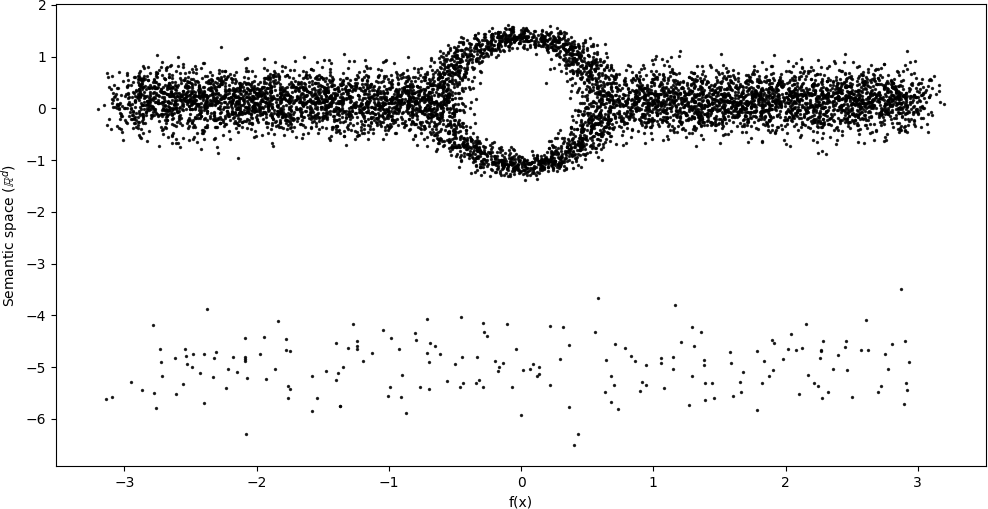}
    \caption{Synthetic data with two components with imbalanced density. The higher-density component has a genus-1 feature.}%
    \label{fig:data1}%
\end{figure}

Our test dataset needs to have a high density and low density region, and the high density region needs a smaller-scale topological feature to try to detect. The test dataset, displayed in Figure \ref{fig:data1}, comprises of a high-density component with a 1-dimensional topological feature (a loop), and a lower-density component. If the lens-space density was more uniform across the dataset, we would expect Mapper to output a correct graph across a wide range of parameters. Figure \ref{fig:balanced-genus1} displays a version of this dataset with uniform density and a graph produced by running Mapper on it. The persistent homology of the datasets in Figures \ref{fig:data1} and \ref{fig:balanced-genus1} are the same, but with different distribution of points. This is why we expect Mapper and density-based Mapper to produce graphs similar to the right-hand side of Figure \ref{fig:balanced-genus1} on both datasets. For this experiment, a Mapper graph is considered \emph{correct} if it meets two criteria:
\begin{enumerate}
    \item It consists of two connected components,
    \item One of these components contains a single genus-1 feature.
\end{enumerate}

\begin{figure}
    \centering
    \includegraphics[width=\linewidth]{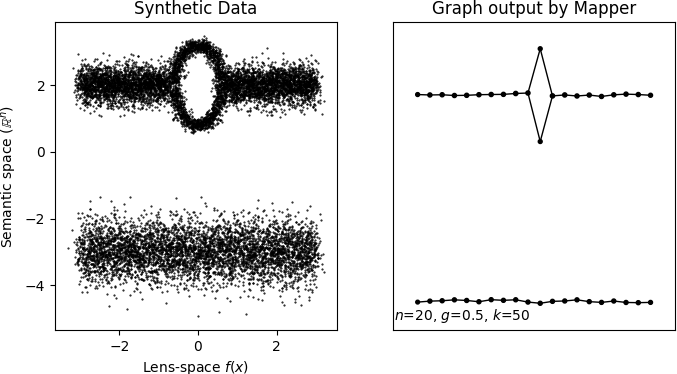}
    \caption{Synthetic data with two components of similar density, and the combinatorial graph produced by running Mapper on this data.}
    \label{fig:balanced-genus1}
\end{figure}

Selecting parameters to have standard mapper Mapper produce a correct graph for this dataset is difficult. If the resolution of the gomic $~\cU$ of the lens-space is too low, the resulting sets will contain very few points from the sparse region of the dataset, preventing Mapper from detecting its connected component. Conversely, if the resolution is too high, the algorithm fails to resolve the genus-1 feature in the dense component.

Density-based Mapper addresses this trade-off by using the relative densities of the components to adjust sets in the kerneled cover, widening sets in the low density region and narrowing them in the high density region. This should enable density-based Mapper to detect both the genus-1 feature and the sparse component across a broader parameter range than standard Mapper. 

To test this hypothesis, we ran Mapper and density-based Mapper on this dataset with a range of parameters. Specifically, we varied the parameters that affect the resolution of the cover: the number of open sets $N$, and their overlap $g$. Each trial used a Morse-spaced cover (Def. \ref{def:tscover}), and we used a square kernel for density-based Mapper. Figures \ref{fig:g1-map-dbscan} (standard Mapper) and \ref{fig:g1-dbmap-dbscan} (density-based Mapper) show the results using DBSCAN for clustering. The graphs' vertices were positioned at the average position of the data points in their corresponding cluster. Correct outputs are highlighted with green boxes.

Standard Mapper successfully recovered the topology in 3 out of 25 trials, while density-based Mapper succeeded in 10 out of 25 trials. Replacing DBSCAN with HDBSCAN further improved density-based Mapper's performance to 19/25 successful trials (Figure \ref{fig:g1-dbmap-hdbscan}). In contrast, standard Mapper performs much worse with HDBSCAN; we could not get Mapper to recover the correct topology with HDBSCAN for any choice of parameters. Code to reproduce these tests is provided in Appendix A.

\begin{figure}[H]
    \centering
    \includegraphics[width=0.8\textwidth]{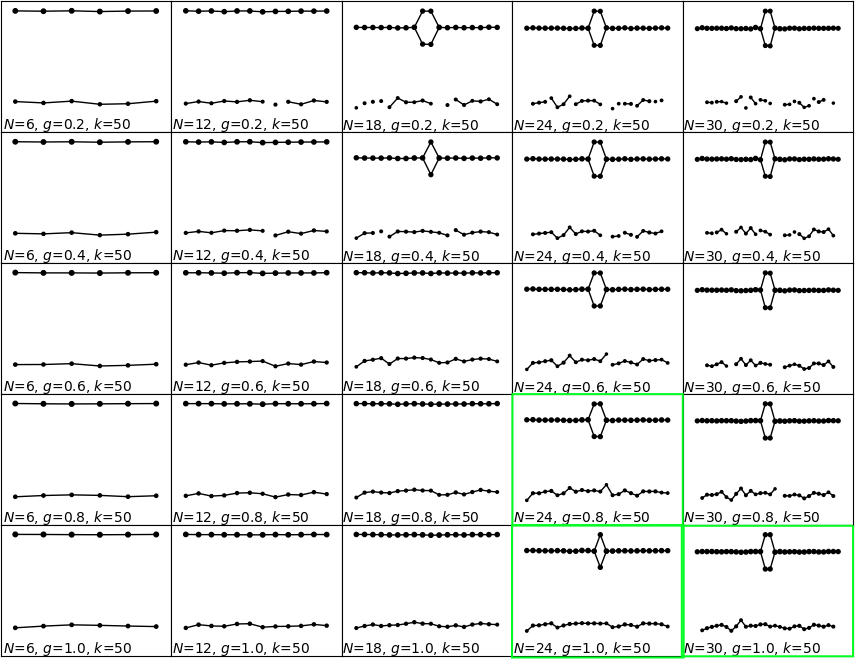}
    \caption{Mapper outputs for a range of parameter choices, with DBSCAN clustering.}%
    \label{fig:g1-map-dbscan}%
\end{figure}
\begin{figure}[H]
    \centering
    \includegraphics[width=0.8\textwidth]{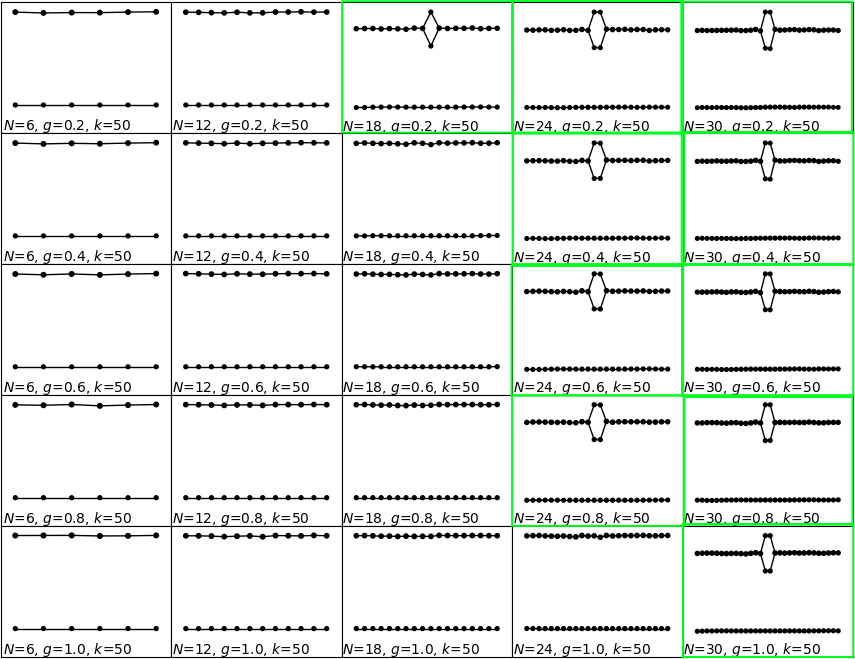}
    \caption{DBMapper outputs for a range of parameter choices, with DBSCAN clustering. }%
    \label{fig:g1-dbmap-dbscan}
\end{figure}
\begin{figure}[H]%
    \centering
    \includegraphics[width=0.8\textwidth]{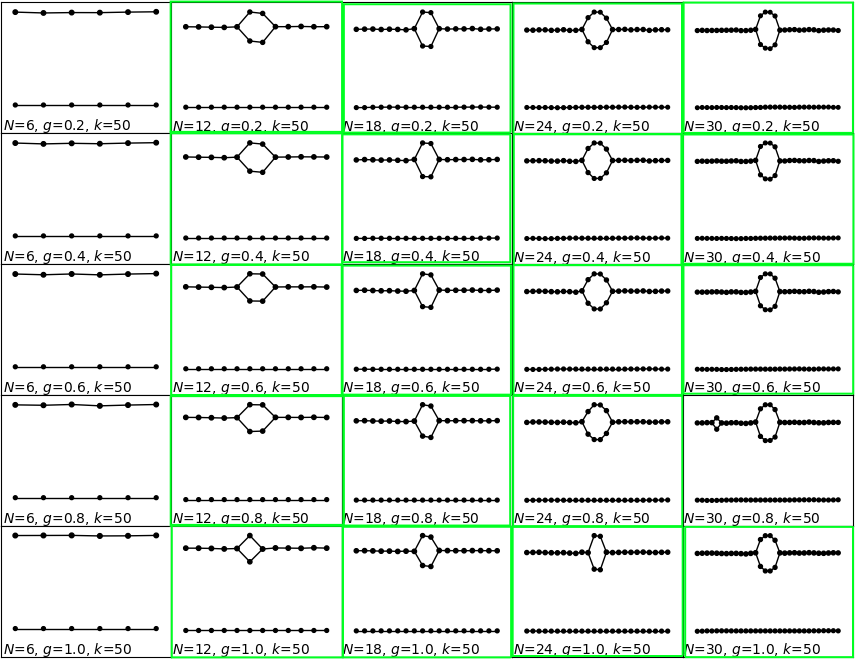}
    \caption{DBMapper outputs for a range of parameter choices, with HDBSCAN clustering.}%
    \label{fig:g1-dbmap-hdbscan}%
\end{figure}
\pagebreak

These figures also demonstrate the failure modes of Mapper and density-based Mapper. Both variants of Mapper can fail to find topological components, such as on the left side of all three figures where the graphs are missing the 1-dimensional component of the dataset. They can also find spurious topological components. This occurs in the top-right of Figure \ref{fig:g1-map-dbscan}, where there are many extra connected components, and for $N=30, g=0.8$ in Figure \ref{fig:g1-dbmap-hdbscan}, where there is an extra 1-dimensional component.

\section{Conclusions and Future Research}
\label{sec:conclusions}

Although our results are promising, there are many potential directions in which our work could be improved. Let us recapitulate with a summary of our proposal and conclusions, and then suggest avenues for future research.

\subsection{Conclusions}

In this work, we proposed the density-based Mapper algorithm for computing the persistent homology of a dataset $\XX_n\subset X$ with Morse-type function $f$. The two key ideas of the algorithm are:
\begin{enumerate}
    \item Generalise the open cover of $X$ to a kerneled cover.
    \item Scale the resolution of the kerneled cover proportionally to the local lens-space density of the dataset.
\end{enumerate}
The novel results of this work are:
\begin{enumerate}
    \item Partial generalisation of the work of \cite{carriere_structure_2018} to kerneled covers, which include:
    \begin{enumerate}
        \item Covers whose resolution varies in $X$;
        \item Covers by fuzzy sets.
    \end{enumerate}
    \item Improvement to Mapper's robustness for datasets with widely-varying lens-space density.
\end{enumerate}

First, we ensured that density-based Mapper preserves some of the important statistical guarantees that make Mapper appealing. In particular, we proved Theorem \ref{thm:main-bound}, which provides an upper bound on the bottleneck distance between density-based Mapper's result and the Reeb graph of $X$. 

Additionally, we provided preliminary experimental evidence that density-based Mapper improves upon standard Mapper for datasets that have widely-varying lens-space density. Specifically, our experiment found there is a wider range of resolution and overlap parameters where density-based Mapper will detect the correct topology compared to standard Mapper.

These changes make it easier to choose parameters for density-based Mapper. In a typical application, one does not know the topology of the dataset in advance. Therefore, it is not straightforward to know when Mapper has computed the correct persistent homology. There are heuristics to check this, such as comparing the results of Mapper across a range of parameters. If the features in the output are invariant under changes of parameters, we can be more confident that they come from the topology of the data, instead of being artifacts of the algorithm. By making density-based Mapper more robust to change of parameters, our proposal makes it easier to obtain output that is invariant to changes of parameters.

In addition to this article, we have released a reference implementation which is geared towards application to temporal topic modelling, accessible at:

{\hfill \href{https://github.com/tutteinstitute/temporal-mapper}{https://github.com/tutteinstitute/temporal-mapper}. \hfill}

\subsection{Future Research}

This work was motivated by temporal topic modelling, and our initial attempts to apply density-based Mapper to our problems of interest revealed many new challenges to solve. The results of our initial applied experiments are inconclusive -- density-based Mapper produces reasonable output graphs, however it is difficult to assess their correctness or to visualize them. To obtain meaningful analysis of temporal topic information using density-based Mapper, more research is required on extracting and visualizing features of interest from the output graph. 

Another avenue for improvement in applications is to combine this work with other efforts to robustify Mapper. In particular, $G$-Mapper \citep{alvarado_g-mapper_2023} provides a method for selecting the resolution $r$. We suspect that using their method to select an initial resolution, then varying it according to lens-space density, will be a highly effective way of choosing Mapper parameters.

There are also interesting theoretical questions remaining. Recently, it was shown that for every graph, there is a choice of Morse-type function $f$ and cover $\cU$ that exhibits the graph as a Mapper graph \citep{alvarado_any_2024}. Our work suggests that for a fixed $f$, there may be graphs which are not Mapper graphs, but are density-based Mapper graphs. It would be interesting to find such an example, or conversely a construction of Mapper parameters for the given $f$ which can produce any given density-based Mapper graph.

Furthermore, when using a clustering algorithm that permits weighting the input points, the shape of kernel used in density-based Mapper has some effect on the output graph, and this effect is not well-understood. More research is required to understand this relationship, and to determine the practical trade-offs of using different kernel shapes.

Finally, for datasets $\XX_n\subset X$ whose points are uniformly distributed in $L=f(X)$, density-based Mapper reduces to regular Mapper. Thus, it may be possible to prove that density-based Mapper is equivalent to a composition of a uniformization process similar to UMAP followed by Mapper. If this turns out to be true, it may suggest a natural way to generalize density-based Mapper to $d$-dimensional lens functions.

\acks{Funding supporting this work: Province of Ontario Graduate Scholarship. K.R. was employed by the Government of Canada during part of this work.}

\bibliography{dbmapper.bib}

\appendix
\section{Code to Reproduce Computational Experiments.}
In this appendix we've included python 3 code snippets which can be used to reproduce the 
computational experiments described in Section \ref{sec:experiments}. 
\subsection{Synthetic Data Generation}
The code in this section was used to generate the synthetic data in the experiments. A working environment should require only python 3, and recent versions of numpy and matplotlib. 

The following code block generates Figure \ref{fig:data-pullback}.
\begin{lstlisting}[language=Python]
import numpy as np
from numpy.random import randn
import matplotlib.pyplot as plt

def generate_consistent(n_pts, loc, width):
    # Sample a component centered at `loc`, with width `width`
    x,y=loc
    widthx,widthy=width
    x_axis = widthx * np.random.randn(n_pts,) + x
    y_axis = widthy * np.random.randn(n_pts,) + y
    t_axis = 10*np.random.rand(n_pts,)

    cluster_consistent = np.array([
        x_axis,
        y_axis,
        t_axis,
    ])
    return cluster_consistent.T


# Generate three components with differing number of points
cpt1 = generate_consistent(200, (3,0), (0.3,0.2))
cpt2 = generate_consistent(600, (0,0), (0.3,0.2))
cpt3 = generate_consistent(1800, (-3,0), (0.3,0.2))

# Concatenate the arrays to save as a .npy file
cluster_merge = np.concatenate((cpt1,cpt2,cpt3),axis=0)
cluster_merge = np.squeeze(cluster_merge)
intended_clusters = [0]*200 + [1]*600 + [2]*1800
np.save( "./density_demo_data.npy", cluster_merge)

# Generate Figure 4 
fig = plt.figure()
ax = fig.add_subplot(projection='3d')
ax.scatter(
    cpt1[:,0],
    cpt1[:,1],
    cpt1[:,2],
    c='black',
    label="Low-density",
    marker="x",
)
ax.scatter(
    cpt2[:,0],
    cpt2[:,1],
    cpt2[:,2],
    c='red',
    label="Medium-density",
    marker='.'
)
ax.scatter(
    cpt3[:,0],
    cpt3[:,1],
    cpt3[:,2],
    c='Blue',
    label="High-density",
    marker="1",
)

ax.set_ylim(-3,3)
ax.set_xlabel("x-axis")
ax.set_xlim(-3,3)
ax.set_ylabel("y-axis")
ax.set_zlim(0,10)
ax.set_zlabel("z-axis")
ax.legend()
plt.show()
\end{lstlisting}

The following code block generates Figure \ref{fig:data1}.
\begin{lstlisting}[language=Python]
import numpy as np
from numpy.random import randn
import matplotlib.pyplot as plt

# The data will be generated in three pieces
# - The loop 
# - The rest of the component containing the loop
# - The second component

# Generate the loop:
angles = np.random.randn(2000,)
circle = np.array([
    0.6*np.cos(2*np.pi*angles), 1.2*np.sin(2*np.pi*angles)+2
]).T
circle += 0.1*np.random.randn(2000,2) 
left=np.amin(circle[:,0])+0.13
right=np.amax(circle[:,0])-0.13

# Generate the rest of the loop's component
handle_leftx = np.linspace(-3,left,2500)+ 0.1*np.random.randn(2500,)
handle_rightx = np.linspace(right,3,2500)+ 0.1*np.random.randn(2500,)
handle_x = np.hstack((handle_leftx,handle_rightx))
handle = np.array([
    handle_x, 0.3*np.random.randn(5000,)+2
]).T
dense = np.vstack((handle, circle)) # combine the dense component into one array

# Generate the sparse component
sparse_x = np.linspace(-3,3,200)
sparse_x += 0.1*np.random.randn(200,)
sparse = np.array([
   sparse_x, 0.5*np.random.randn(200,)-3
]).T

# Combine arrays and save as .npy
points = np.vstack((dense, sparse))
np.save("./genus1_demo.npy",points)

# Generate Figure
fig, ax = plt.subplots(1)
ax.set_title("Generated Data")
ax.scatter(
    points[:,0], points[:,1],
    marker='.',
    s=1,
)
ax.set_xlabel(r'Lens-Space $f(x)$')
ax.set_ylabel(r'Semantic space ($\mathbb{R}^n)$')
plt.show(fig)
\end{lstlisting}

\subsection{Running the Reference Implementation}
Here we provide code snippets used to run our reference implementation of DBMapper in the experiments. The reference implementation can be downloaded from GitHub, at (link here). For this article, the version used is v 0.4.0, which can be obtained from the Releases section of the GitHub page.

The following code block generates Figure \ref{fig:pullback}.
\begin{lstlisting}[language=Python]
import numpy as np
import pandas as pd
import datamapplot
import sys
import os
import networkx as nx
import temporalmapper as tm
import temporalmapper.utilities_ as tmutils
import temporalmapper.weighted_clustering as tmwc

import matplotlib.pyplot as plt
from matplotlib.colors import to_rgb
from sklearn.decomposition import PCA
from sklearn.neighbors import NearestNeighbors
from mpl_toolkits.axes_grid1 import make_axes_locatable
from sklearn.metrics import pairwise_distances
from sklearn.cluster import DBSCAN

data_time = np.load("data/density_test_data.npy")
data = data_time[:,0:2]
time = data_time[:,2]
sorted_indices = np.argsort(time)
time = time[sorted_indices]
data = data[sorted_indices]
N_data = np.size(time)


# Construct the temporal graph.
map_data = data
y_data = PCA(n_components=1).fit_transform(data)
clusterer = DBSCAN()
N_checkpoints = 20
kernel_params = (1),
TG = tm.TemporalMapper(
    time,
    map_data,
    clusterer,
    N_checkpoints = N_checkpoints,
    neighbours = 150,
    slice_method='time',
    overlap=1,
    rate_sensitivity=1,
    kernel=tmwc.square,
    verbose=True
)
TG.build()

idx=15
slice_ = (TG.weights[idx] >= 0.1).nonzero()
cp_with_ends = [np.amin(time)]+list(TG.checkpoints)+[np.amax(time)]
bin_width = (cp_with_ends[idx+1]-cp_with_ends[idx])
fig, (ax1,ax2) = plt.subplots(1,2)
fig.set_figwidth(10)
ax1.set_title(f"$f$-Density")
sca=ax1.scatter(time,y_data,s=1,c=TG.density)
ax2.scatter(time,y_data,s=1,c='grey')
divider = make_axes_locatable(ax1)
cax = divider.append_axes('right', size='5%', pad=0.05)
fig.colorbar(sca, cax=cax, orientation='vertical')
ax2.scatter(time[slice_],y_data[slice_],s=1,c='red')
tstr = f'Pullback of $({TG.checkpoints[idx]-(bin_width/2)*(1+TG.g):.2f},{TG.checkpoints[idx]+bin_width/2*(1+TG.g):.2f})$'
ax2.set_title(tstr)
ax2.set_xlabel("Time")
ax1.set_xlabel("Time")
ax1.set_ylabel("PCA to 1d")
ax2.axvline(TG.checkpoints[idx]+(bin_width/2)*(1+TG.g),c='k')
ax2.axvline(TG.checkpoints[idx]-(bin_width/2)*(1+TG.g),c='k')
plt.savefig("td-verify.png")
plt.show()
\end{lstlisting}

The following code block generates Figures \ref{fig:g1-map-dbscan}-\ref{fig:g1-dbmap-hdbscan}.
\begin{lstlisting}[language=Python]
import numpy as np
import pandas as pd
import datamapplot
import sys
import os
import networkx as nx
import temporalmapper as tm
import temporalmapper.utilities_ as tmutils
import temporalmapper.weighted_clustering as tmwc
import matplotlib.pyplot as plt
from matplotlib.colors import to_rgb
from sklearn.decomposition import PCA
from sklearn.neighbors import NearestNeighbors
from sklearn.cluster import DBSCAN
from mpl_toolkits.axes_grid1 import make_axes_locatable
import matplotlib as mpl
from tqdm import trange

data_time = np.load("data/genus1_demo.npy")
data_unsort = data_time[:,1].T
timestamps_unsort = data_time[:,0].T
sorted_indices = np.argsort(timestamps_unsort)
data = data_unsort[sorted_indices]
timestamps = timestamps_unsort[sorted_indices]
N_data = np.size(timestamps)
map_data = y_data = data

def generate_plot(TG, label_edges = True,ax=None, threshold = 0.2, vertices = None):
    self = TG
    if ax is None:
        ax = plt.gca()
    if type(vertices) == type(None):
        vertices = self.G.nodes()

    G = self.G.subgraph(vertices)
    pos = {}
    slice_no = nx.get_node_attributes(TG.G, 'slice_no')
    for node in vertices:
        t = slice_no[node]
        pt_idx = TG.get_vertex_data(node)
        w = TG.weights[t,pt_idx]
        node_ypos = np.average(np.squeeze(TG.data[pt_idx]),weights=w)
        node_xpos = t #np.average(TG.time[pt_idx],weights=w)
        pos[node] = (node_xpos, node_ypos)

    edge_width = np.array([d["weight"] for (u,v,d) in G.edges(data = True)])
    elarge = [(u, v) for (u, v, d) in G.edges(data=True) if d["weight"] >= threshold]
    esmall = [(u, v) for (u, v, d) in G.edges(data=True) if 0.1< d["weight"] < threshold]
    nx.draw_networkx_edges(G, pos, ax=ax, edgelist=elarge, width=1, arrows=False)
    if label_edges:
        edge_labels = nx.get_edge_attributes(G, "weight")
        nx.draw_networkx_edge_labels(G, pos, edge_labels)

    node_size = [np.log2(np.size(self.get_vertex_data(node))) for node in vertices]
    clr_dict = nx.get_node_attributes(self.G, 'cluster_no')
    node_clr = [clr_dict[node] for node in vertices]

    nx.draw_networkx_nodes(G, pos, ax=ax,node_size=node_size, node_color=node_clr)
    return ax

"""
Running standard mapper over a range of parameters, with DBSCAN.
"""

checkpoint_numbers = [6,12,18,24,30]
overlap_parameters = [0.2,0.4,0.6,0.8,1.]

fig, axes = plt.subplots(5,5)
fig.set_figwidth(11)
fig.set_figheight(8.5)
fig.dpi = 200
axes = axes.reshape(5*5)
clusterer = DBSCAN()
j = 0
for k in trange(25):
    TG = tm.TemporalMapper(
        timestamps,
        map_data,
        clusterer,
        N_checkpoints = checkpoint_numbers[k%5],
        neighbours = 50,
        overlap = overlap_parameters[j],
        slice_method='time',
        rate_sensitivity=0,
        kernel=tmwc.square,
        #kernel_params=(overlap_parameters[j],),
    )
    TG.build()
    generate_plot(TG,label_edges = False,ax=axes[k])
    xmin,xmax=axes[k].get_xlim()
    ymin,ymax=axes[k].get_ylim()
    axes[k].text(xmin+0.1,ymin+0.1,fr'$n$={TG.N_checkpoints}, $g$={TG.g}, $k$={50}')
    if k%5==4:
        j+=1
plt.subplots_adjust(wspace=0, hspace=0)
plt.savefig("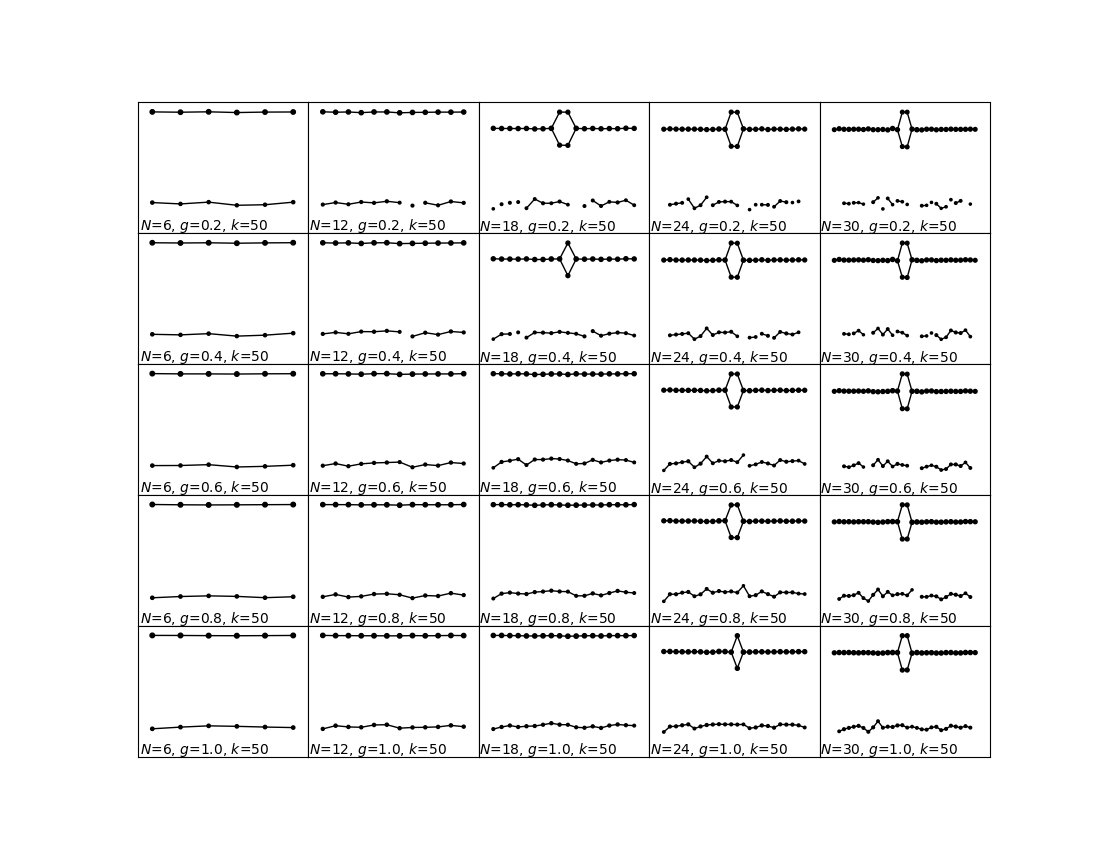")
plt.show()

"""
Running fuzzy mapper over a range of parameters, with DBSCAN.
"""
checkpoint_numbers = [6,12,18,24,30]
overlap_parameters = [0.2,0.4,0.6,0.8,1.]

fig, axes = plt.subplots(5,5)
fig.set_figwidth(11)
fig.set_figheight(8.5)
fig.dpi = 200
axes = axes.reshape(5*5)
clusterer = DBSCAN()
j = 0
for k in trange(25):
    TG = tm.TemporalMapper(
        timestamps,
        map_data,
        clusterer,
        N_checkpoints = checkpoint_numbers[k%5],
        neighbours = 50,
        slice_method='time',
        overlap = overlap_parameters[j],
        rate_sensitivity=1,
        kernel=tmwc.square,
    )
    TG.build()
    generate_plot(TG,label_edges = False,ax=axes[k])
    xmin,xmax=axes[k].get_xlim()
    ymin,ymax=axes[k].get_ylim()
    axes[k].text(xmin+0.1,ymin+0.1,fr'$n$={checkpoint_numbers[k%5]}, $g$={overlap_parameters[j]}, $k$={50}')
    if k%5==4:
        j+=1
plt.subplots_adjust(wspace=0, hspace=0)
plt.savefig("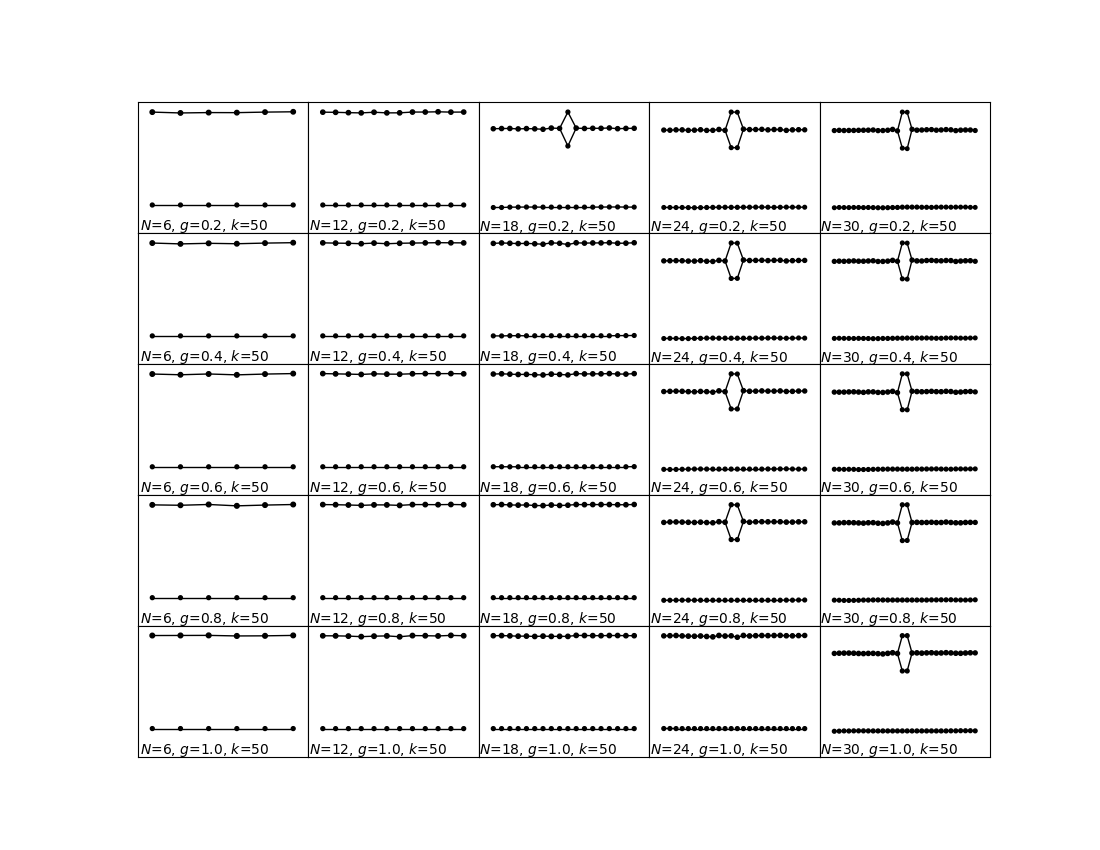")
plt.show()

"""
Running fuzzy mapper over a range of parameters, with HDBSCAN.
"""
checkpoint_numbers = [6,12,18,24,30]
overlap_parameters = [0.2,0.4,0.6,0.8,1.]

fig, axes = plt.subplots(5,5)
fig.set_figwidth(11)
fig.set_figheight(8.5)
fig.dpi = 200
axes = axes.reshape(5*5)
clusterer = HDBSCAN(min_cluster_size=50)
j = 0
for k in trange(25):
    TG = tm.TemporalMapper(
        timestamps,
        map_data,
        clusterer,
        N_checkpoints = checkpoint_numbers[k%5],
        neighbours = 50,
        slice_method='time',
        overlap = overlap_parameters[j],
        rate_sensitivity=1,
        kernel=tmwc.square,
    )
    TG.build()
    generate_plot(TG,label_edges = False,ax=axes[k])
    xmin,xmax=axes[k].get_xlim()
    ymin,ymax=axes[k].get_ylim()
    axes[k].text(xmin+0.1,ymin+0.1,fr'$n$={checkpoint_numbers[k%5]}, $g$={overlap_parameters[j]}, $k$={50}')
    if k%5==4:
        j+=1
plt.subplots_adjust(wspace=0, hspace=0)
plt.savefig("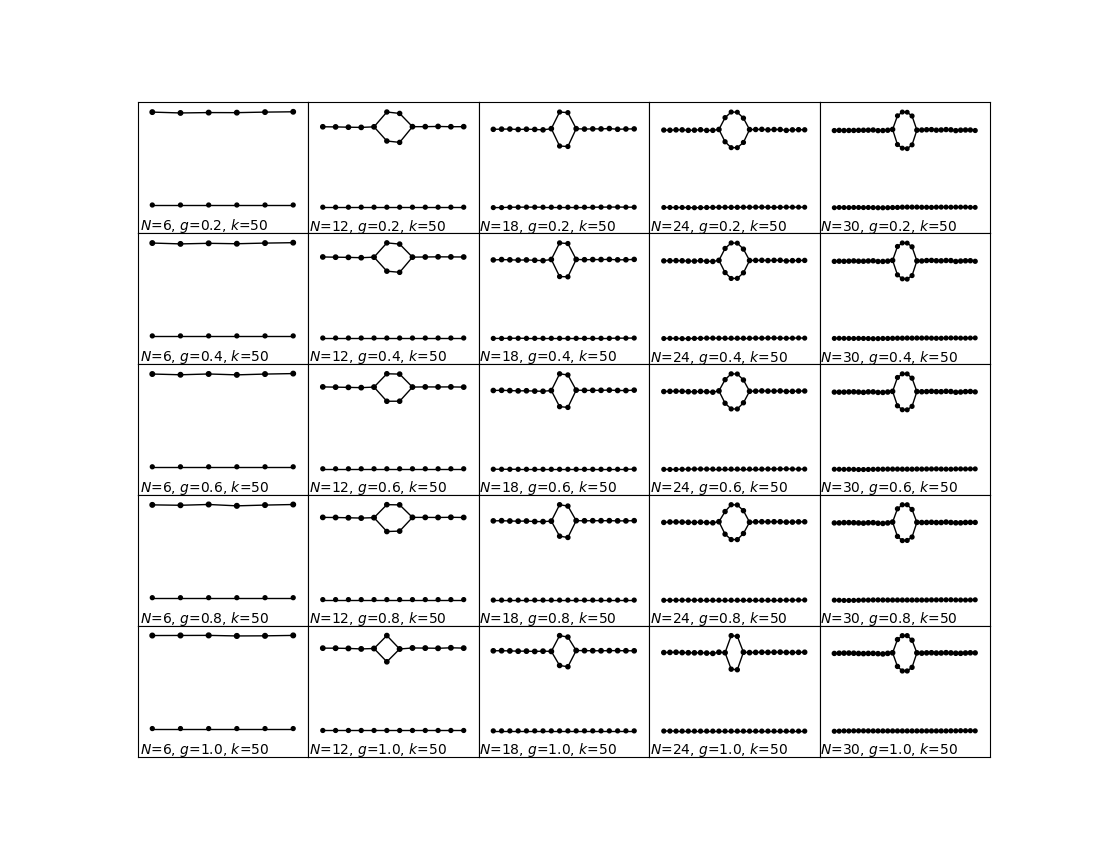")
plt.show()
              
\end{lstlisting}
\end{document}